\documentclass{article} 
\usepackage{microtype}
\usepackage{graphicx}
\usepackage{subfigure}
\usepackage{booktabs}

\usepackage{hyperref}
\usepackage{url}
\usepackage{amssymb,amsmath}
\usepackage{algorithm}
\usepackage{algorithmic}
\usepackage{hyperref}
\usepackage{graphicx}
\usepackage{multirow}
\usepackage{tabularx}

\usepackage{makecell, booktabs, mathtools}

\usepackage{pdflscape}

\usepackage{iclr2016_workshop}

\usepackage{amsthm}
\newtheorem{prop}{Proposition}
\newtheorem{theorem}{Theorem}

\DeclareMathOperator*{\argmax}{arg\,max}

\title{A Review of Learning with Deep Generative Models from Perspective of Graphical Modeling}

\author{Zhijian Ou \\
	Tsinghua University, Beijing, China.\\
	\texttt{ozj@tsinghua.edu.cn} \\
	\today
}

\begin{document}

\maketitle
	
%

\newcommand{\fix}{\marginpar{FIX}}
\newcommand{\new}{\marginpar{NEW}}


\begin{abstract}
	
This document aims to provide a review on learning with deep generative models (DGMs), which is an highly-active area in machine learning and more generally, artificial intelligence. 
This review is not meant to be a tutorial, but when necessary, we provide self-contained derivations for completeness.
This review has two features.
First, though there are different perspectives to classify DGMs, we choose to organize this review from the perspective of graphical modeling, because the learning methods for directed DGMs and undirected DGMs are fundamentally different.
Second, we differentiate model definitions from model learning algorithms, since different learning algorithms can be applied to solve the learning problem on the same model, and an algorithm can be applied to learn different models.
We thus separate model definition and model learning, with more emphasis on reviewing, differentiating and connecting different learning algorithms.
We also discuss promising future research directions.\footnote{This review is by no means comprehensive as the field is evolving rapidly. The authors apologize in advance for any missed papers and inaccuracies in descriptions. Corrections and comments are highly welcome.}

\end{abstract}

\section{Introduction}

Generative models can be used for compression, denoising, inpainting, synthesis, recognition, unsupervised feature learning, semi-supervised learning, and many other tasks.
Learning with probabilistic generative models is one of the core research problems in machine learning and more generally, artificial intelligence.

Generative modeling generally refers to the approach that defines the joint distribution of random variables in the studied phenomenon.
In terms of graphical modeling terminology \cite{koller2009probabilistic}, probabilistic models can be broadly classified into two classes - directed and undirected\footnote{An easy way to tell an undirected model from a directed model is that an undirected model involves the normalizing constant (also called the partition function in physics), while the directed model is self-normalized.}.
In the \textbf{directed} graphical models (also known as Bayesian networks),  the joint distribution is factorized into a product of local conditional density functions, whereas in the \textbf{undirected} graphical models (also known as Markov random fields or Markov networks) the joint density function is defined to be proportional to the product of local potential functions. Simply speaking, an easy way to tell an undirected model from a directed model is that an undirected model involves the normalizing constant (also called the partition function in physics), while the directed model is self-normalized.

In recent years, generative modeling techniques have been greatly advanced by inventing new models with new learning algorithms.
\textbf{Deep generative models (DGMs)} generally refer to models with multiple layers of stochastic or deterministic variables.
A large class of DGMs is defined by using multiple stochastic hidden layers, such as the early Sigmoid Belief Networks (SBNs) \cite{Neal1992ConnectionistLO,Saul1996}, Helmholtz machines (HMs) \cite{Hinton1995the,dayan1995helmholtz}, and recently Deep Belief Networks (DBNs) \cite{hinton2006a}, Deep Boltzmann Machines (DBMs) \cite{dbm}.
In a recent trend, another class of DGMs is defined by utilizing neural networks in model representation, especially using neural networks with multiple (deterministic) hidden layers, often called deep neural networks (DNNs).
The implementation is simple, which basically is to use DNNs to define the  conditional densities in directed models or the potentials in undirected models.
Note that compared to modeling with multiple deterministic layers, modeling with multiple stochastic layers presents much greater challenge for model learning and thus yields inferior performance.
This is observed in both directed and undirected models.

Before we dive into the low-level presentation, it is worthwhile to emphasize some high-level concepts in the content organization in this review.
\begin{itemize}	
\item It is clear that the division of directed models and undirected models and the division of using deterministic and stochastic layers are orthogonal perspectives to specify DGMs.
In the following, we choose to organize our review of learning DGMs with respect to the first perspective, because the learning methods for directed DGMs and undirected DGMs are fundamentally different.

\item It is important to differentiate model definitions from model learning algorithms, since different learning algorithms can be applied to solve the learning problem on the same model, and an algorithm can be applied to learn different models.
However, many branding ``models'' in the literature are a combination of the model definition and the model learning. For example, the combination of the particular model definition through DNNs and the variational learning method together is referred to as a variational autoencoder (VAE) \cite{kingma2014auto-encoding,Kingma2014SemiSupervisedLW}.
So we separate model definition and model learning, with more emphasis on reviewing, differentiating and connecting different learning algorithms.

\end{itemize}

\cite{frey2005comparison} provides a good introduction to the classic inference and learning algorithms in graphical models, prior to the proliferation of deep learning.
There exist several excellent reviews related to DGMs.
The Chapter 20 in the book \cite{DLbook} provides an excellent introduction of DGMs, but organizes around the ``models'', which is different from this review.
\cite{Bengio2013RepresentationLA} introduce DGMs, both directed and undirected, in the context of representation learning, but does not cover the recent development, such as VAEs and GANs.
\cite{Hu2017OnUD} mainly presents a unified view of GANs and VAEs and links them to the wake-sleep algorithm, but does not cover undirected models.
Therefore, this review has its own contribution in distinctive content organization and up-to-date introduction.

\textbf{Main symbols}.

$x$ : the observation variable.

$h$ : the hidden variable (or say latent code);

$y$ : the class label;

$p_0(\cdot)$ : the (unknown) true density;

$\tilde{p}(\cdot)$ : the empirical density;

$p_{\theta}(\cdot)$ : the (target) model density with parameter $\theta$;

$q_{\phi}(\cdot)$ : the auxiliary density introduced in training with parameter $\phi$;

$D_{\psi}(\cdot), V_{\psi}(\cdot)$, etc. : some auxiliary functions introduced in training with parameter $\psi$. $D_{\psi}(x)$ denotes the discrminator (with sigmoid output) in adversarial learning.
$V_\psi(x) = log \frac{D_\psi(x)}{1-D_\psi(x)}$ is the logit of $D_\psi(x)$, and reversely $D_\psi(x) = \frac{1}{1+e^{-V_\psi(x)}} = \sigma( V_\psi(x) )$.

$\sigma(v) \triangleq \frac{1}{1+e^{-v}} $ denotes the sigmoid function.

$\mathcal{L}(x;\theta,\phi)$ : the variational lower bound (V-LB) or the evidence lower bound (ELBO) on the marginal log-likelihood $log p_\theta(x)$, defined in Eq. \ref{eq:VLB1}.
The V-LB pooled over the training data is denoted by $\mathcal{L}(\theta,\phi) \triangleq \sum_{k=1}^{n} \mathcal{L}(x_k;\theta,\phi)$.

$\mathcal{I}_K(x;\theta,\phi)$ : the importance-weighting lower bound (IW-LB) on the marginal log-likelihood $log p_\theta(x)$, using $K$ samples, defined in Eq. \ref{eq:IW-LB}.
Note that the IW-LB in the special case of $K=1$ is equal to the standard V-LB: $\mathcal{I}_1(x;\theta,\phi) = \mathcal{L}(x;\theta,\phi)$.
The IW-LB pooled over the training data is denoted by $\mathcal{I}_K(\theta,\phi) \triangleq \sum_{k=1}^{n} \mathcal{I}_K(x_k;\theta,\phi)$.

$\mathcal{F}_f(\theta,\psi)$ : the objective function in learning $f$-GANs, which, as defined in Eq. \ref{eq:f-gan-obj2}, is a variational lower bound on the $f$-divergence $D_f\left[p_0 || p_\theta\right]$. The subscript $f$ denotes the $f$-function used in defining the $f$-divergence.

$\mathcal{J}(\hat\theta)$ : the objective function in NCE learning, as defined in Eq \ref{eq:NCE-obj}. $\hat\theta = (\theta, c)$ denotes the new parameter vector, which consists of the log normalization constant $c$ as a parameter.


\textbf{Notations}.

For any generic sequence $\left\lbrace z_n\right\rbrace $ we shall use $z_{i:j}$ to denote $z_i, z_{i+1}, \cdots\, z_j$.
Similarly, wherever a collection of indices appears in the subscript, we refer to the corresponding collection of indexed variables, e.g. $c_{l,1:H} \triangleq \left\lbrace  c_{l,1}, c_{l,2}, \cdots\, c_{l,H} \right\rbrace$.

The entropy is defined as $H[q] \triangleq - \int q log q$.

The inclusive KL-divergence between two distributions $p(\cdot)$ and $q(\cdot)$ is defined as $KL[p||q] \triangleq \int p log \left( \frac{p}{q} \right) $, which by default is called the KL-divergence, and is sometimes referred to as the forward KL-divergence.

The exclusive KL-divergence is defined as $KL[q||p] \triangleq \int q log \left( \frac{q}{p} \right) $, which is also referred to as the reverse KL-divergence.

The Jensen-Shannon divergence is defined as
\begin{displaymath}
JS[p||q] \triangleq \frac{1}{2}KL\left[p||\frac{p+q}{2}\right] + \frac{1}{2}KL\left[q||\frac{p+q}{2}\right]
\end{displaymath}

The symmetric KL divergence is defined as
\begin{displaymath}
KL_{sym}[p||q] \triangleq \frac{KL(p||q)+KL(q||p)}{2}
\end{displaymath}

\textbf{Mathematical Properties}.

Formally, for any density function $p_X(x;\theta)$, the partial derivative w.r.t. $\theta$ of the logarithm of the density function, $\frac{\partial}{\partial \theta} log p_X(x;\theta)$, is called the ``score''. Under certain regularity conditions, the expectation of the score w.r.t. the density itself is 0. This formula is often referred in presenting Fisher information\footnote{\url{https://en.wikipedia.org/wiki/Fisher_information}}, so we call it \textbf{Fisher Equality}, which, as we will show, is frequently used in this review.
\begin{displaymath}
E_{p_X(x;\theta)}\left[  \frac{\partial}{\partial \theta} log p_X(x;\theta) \right] = 0.
\end{displaymath}

\section{Background}

\begin{algorithm}[tb]
	\caption{The general stochastic approximation (SA) algorithm}\label{alg:SA}
	\begin{algorithmic}	
		\FOR {$t=1,2,\cdots$}
		\STATE
		\begin{enumerate}
			\item Draw a sample $z^{(t)}$ with a Markov transition kernel $K_{\lambda^{(t-1)}}(\cdot|z^{(t-1)})$, which starts with $z^{(t-1)}$ and admits $p(\cdot; \lambda^{(t-1)})$ as the invariant distribution.
			\item Set $\lambda^{(t)} = \lambda^{(t-1)} + \gamma_t F(z^{(t)};\lambda^{(t-1)}) $, where $\gamma_t$ is the learning rate.
		\end{enumerate}	
		\ENDFOR
	\end{algorithmic}
\end{algorithm}

\begin{algorithm}[tb]
	\caption{SA with multiple moves}\label{alg:SA-multiple-move}
	\begin{algorithmic}	
		\FOR {$t=1,2,\cdots$}
		\STATE
		\begin{enumerate}
	\item Set $z^{(t,0)}=z^{(t-1,K)}$.
	For $k$ from $1$ to $K$,
	generate $z^{(t,k)} \sim K_{\lambda^{(t-1)}}(\cdot|z^{(t, k-1)})$,
	where $K_{\lambda^{(t-1)}}(\cdot|\cdot)$ is a Markov transition kernel that admits $p(\cdot; \lambda^{(t-1)})$ as the invariant distribution.
	\item Set $\lambda^{(t)} = \lambda^{(t-1)} + \gamma_t \{ \frac{1}{K} \sum_{z\in B^{(t)}} F(z;\lambda^{(t-1)}) \}$,  where $B^{(t)} = \{ z^{(t,k)} | k = 1,\cdots,K \}$.
		\end{enumerate}	
		\ENDFOR
	\end{algorithmic}
\end{algorithm}

We begin this review with a brief introduction to \textbf{stochastic approximation (SA)} \citep{SA51}, which lays the mathematical foundation for many recent learning methods.
SA methods are an important family of iterative stochastic optimization algorithms, introduced in \cite{SA51}.
There has since been a vast literature on theory, methods, and applications of SA \cite{benveniste2012adaptive,chen2002stochastic}.
In fact, the widely used variational learning and adversarial learning are applications of the SA methodology.

Basically, stochastic approximation provides a mathematical framework for stochastically solving a root finding problem, which has the form of expectations being equal to zeros.
Suppose that the objective is to find the solution $\lambda^*$ of $f(\lambda) = 0$ with
\begin{equation}
\label{eq:SA}
f(\lambda) = E_{z \sim p(\cdot; \lambda) } [ F(z;\lambda) ],
\end{equation}
where $\lambda$ is a $d$-dimensional parameter vector in $\Lambda \subset R^d$, and $z$ is an observation from a probability distribution $p(\cdot; \lambda)$ depending on $\lambda$.
$F(z;\lambda) \in R^d $ is a function of $z$, providing $d$-dimensional noisy measurements of $f(\lambda)$.
Intuitively, we solve a system of simultaneous equations, $f(\lambda) = 0$, which consists of $d$ constraints, for determining $d$-dimensional $\lambda$.

Given some initialization $\lambda^{(0)}$ and $z^{(0)}$, a general SA algorithm iterates as shown in Algorithm \ref{alg:SA} \citep{song2014weak}.
During each SA iteration, it is possible to generate a set of multiple observations $z$ by performing the Markov transition repeatedly 
and then use the average of the corresponding values of $F(z;\lambda)$ for updating $\lambda$, which is known as SA with multiple moves as shown in Algorithm \ref{alg:SA-multiple-move} \citep{Wang2017LearningTR}.
This technique can help reduce the fluctuation due to slow-mixing of Markov transitions.

The convergence of SA has been studied under various regularity conditions \cite{benveniste2012adaptive,chen2002stochastic}.
The sufficient conditions for the convergence are often expressed as the requirement for the learning rates (e.g. satisfying that $\sum_{t=0}^\infty \gamma_t = \infty$ and $\sum_{t=0}^\infty \gamma_t^2 < \infty$), together with a few mild technical requirements for $f(\lambda)$.
In practice, we can set a large learning rate at the early stage of learning and decrease to $1/t$ for convergence.
For completeness, we provide a short summary on the convergence of $\left\lbrace \lambda_t, t \ge 1\right\rbrace $ in Algorithm \ref{alg:SA}, based on Theorem 1 in \cite{song2014weak}.

\begin{theorem}
	\label{theorem:SA}
	Let $\left\lbrace \gamma_t\right\rbrace $ be a monotone nonincreasing sequence of positive numbers such that $\sum_{t=1}^\infty \gamma_t = \infty$ and $\sum_{t=1}^\infty \gamma_t^2 < \infty$. Assume that $\Lambda$ is compact and the Lyapunov condition on $f(\lambda)$ and the drift condition on the transition kernel $K_{\lambda}(\cdot|\cdot)$ hold. Then we have: $d(\lambda_t, \mathcal{L}) \to 0$ almost surely as $t \to \infty$, where $\mathcal{L}=\left\lbrace \lambda: f(\lambda)=0 \right\rbrace $ and $d(\lambda, \mathcal{L}) = \inf_{\lambda' \in \mathcal{L}} || \lambda - \lambda' ||$.
\end{theorem}

Particularly, when $f(\lambda)$ corresponds to the gradient of some objective function, then under certain regularity conditions, 
$\lambda^{(t)}$ will converge to the optimal solution (if the objective function is convex) or a local optimum (if the objective not convex).

Perhaps the most familiar application of SA in machine learning literature is the Stochastic Gradient Descent (SGD) technique.
When the objective (and therefore its gradient) is a sum of many terms that can be computed independently, SGD samples one term at a time and follows one noisy estimate of the gradient with a decreasing step size.
Furthermore, it can be easily seen that SGD training with minibatches is an application of SA with multiple moves.

\begin{figure}[htb]
	\centering  
	\includegraphics[width=0.5\textwidth]{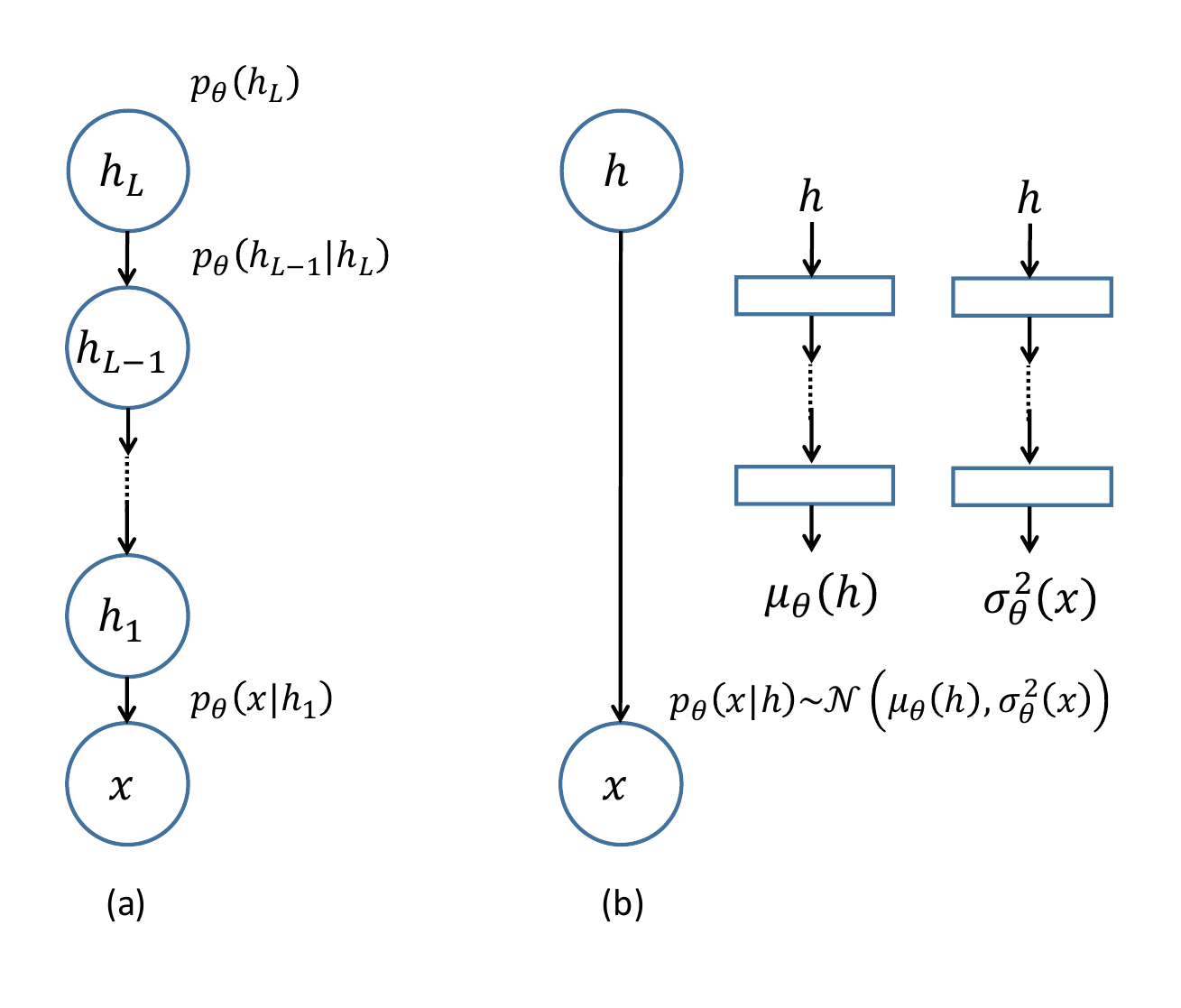}  
	\caption{Graphical model representation of different directed DGMs. Circles represent random variables, while rectangles represent deterministic layers in MLPs. (a) A DGM using multiple stochastic layers, e.g. SBNs, DLGMs. (b) A DGM using multiple deterministic layers but shallow stochastic connections, e.g. VAEs.} 
	\label{fig:directed-model}
\end{figure}  

\section{Learning with Deep Directed Models}

\subsection{Model Definition} \label{sec:directed-model-def}

First, it is useful to make a distinction between prescribed probabilistic models and implicit probabilistic models \cite{mohamed2016learning}.
\textbf{Prescribed models} are those that provide an explicit parametric specification of the distribution of the observed random variable $x$, specifying the likelihood function $p_{\theta}(x)$ with parameter $\theta$. 
Most models in machine learning are of this form. Examples include SBNs \cite{Neal1992ConnectionistLO,Saul1996}, Deep Latent Gaussian Models (DLGMs) \cite{Rezende2014StochasticBA}, VAEs \cite{kingma2014auto-encoding}.
Alternatively, \textbf{implicit models} are defined by a stochastic procedure (a simulation process) that generates data - we can sample data from its generative process, but we may not have access to calculate its density, and thus are referred to as likelihood-free.
A popular example of implicit models are GANs \cite{goodfellow2014generative}: samples from a simple distribution - such as uniform or Gaussian - are transformed nonlinearly and non-invertably by a DNN.

In the following, we roughly classify existing directed DGMs into four classes, on top of the division of prescribed models and implicit models.

\subsubsection{Prescribed models using deep stochastic layers}
\label{sec:model-def-SBN}
This class of models defines a generative process in terms of ancestral sampling through a cascade of hidden stochastic layers, as shown in Figure \ref{fig:directed-model}(a).
\begin{displaymath}
\label{eq:SBN}
p_{\theta}(x, h_{1:L}) = p_{\theta}(h_L) p_{\theta}(h_{L-1} | h_L) \cdots p_{\theta}(x | h_1)
\end{displaymath}
where $h_l$ $(l=1,\cdots\,L)$ are the stochastic hidden vectors.
There are two choices for $h_l$, being binary or real-valued.

\begin{itemize}
\item Like in SBNs, we can define $h_l \in  \left\lbrace 0,1\right\rbrace ^{H_l}$ to be binary hidden vectors, and implement $p_{\theta}(h_{l} | h_{l+1})$ as
\begin{equation} \label{eq:SBN-conditional}
\left\{
\begin{split}
p_{\theta}(h_{l} | h_{l+1}) &= \mathcal{FB}[h_{l} | c_{l,1:H}] \triangleq  \prod_{i=1}^{H_l}  \mathcal{B}[h_{l,i} | c_{l,i}], l=1,\cdots,L-1, \\
c_{l,1:H} &= T_l(h_{l+1}), \\
\end{split}
\right.
\end{equation}
where $\mathcal{B}[h_{l,i}|c_{l,i}]$ refers to the Bernoulli distribution with $p(h_{l,i} = 1) = c_{l,i}$, and $\mathcal{FB}[\cdot|\cdot]$ denotes a factorial Bernoulli distribution.
The transformations $T_l$ can be implemented by multi-layer perceptrons (MLPs), as long as the sigmoid activation functions are used at the output layer\footnote{Since the output components of $T_l$ are used as the parameters for the Bernoulli distributions.}. In practice, a simple transformation $T_l$ is usually used:
\begin{displaymath}
T_l(h_{l+1}) = \sigma(W_{l} h_{l+1} + b_{l}), l=1,\cdots,L-1,
\end{displaymath}
where $W_{l} \in \mathcal{R}^{H_l \times H_{l+1}}$ are the connection weights between $h_l$ and $h_{l+1}$, 
$b_{l} \in \mathcal{R}^{H_l}$ denotes the biases, 
and $\sigma(\cdot)$ denotes the component-wise, sigmoid function.

For the top-level prior $p_{\theta}(h_L)$, we can use a factorial Bernoulli distribution:
\begin{displaymath}
p_{\theta}(h_L) = \mathcal{FB}[ h_{L}|c_{L,1:D}], 
\end{displaymath}
where $c_{L,1:D}$ are parameters.

If the observation $x$ is a binary vector, then $p_{\theta}(x | h_1)$ can be defined analogously, like in defining $p_{\theta}(h_{l} | h_{l+1})$ as shown above. For real-valued observation vector $x$, $p_{\theta}(x | h_1)$ is often implemented as a diagonal-covariance guassian, as shown in the following use of real-valued $h_l$.

\item Like in DLGMs, we can define $h_l \in \mathcal{R}^{H_l}$ to be continuous hidden vectors, and implement $p_{\theta}(h_{l} | h_{l+1})$ as
\begin{equation} \label{eq:DLGM-conditional}
\left\{
\begin{split}
h_{l} &= T_l(h_{l+1}) + G_l \xi_l, \\
\xi_l &\sim \mathcal{N}[0,I]
\end{split}
\right.
\end{equation}
where $\xi_l$ represents standard normal vectors, $G_l$ are matrices, and the transformations $T_l$ are represented by multi-layer perceptrons (MLPs).

We can use the standard normal for the top-level prior $p_{\theta}(h_L)$.
The observation $x$ can be generated from any appropriate distribution  $p_{\theta}(x | h_1)$, whose parameters are specified by a transformation of the first latent layer $h_1$. 

\end{itemize}

The above formulation unifies the use of stochastic layers and deterministic layers.
The construction of using $T_l$, $l=1,\cdots,L-1$, generally allows us to use as many deterministic and stochastic layers as needed. But in practice modeling with deep stochastic layers presents greater challenge for model learning than modeling with deep deterministic layers, which has been increasingly used and merit a separate introduction as described in the following.

\subsubsection{Prescribed models using deep deterministic layers}
\label{sec:model-def-VAE}

This class of models often has shallow stochastic connections, but use deep deterministic layers in implementing the conditional distributions.
VAEs are such an example class of models, as shown in Figure \ref{fig:directed-model}(b).

The generative process consists of two steps. First, a latent code $h$ is drawn from some prior distribution $p_{\theta}(h)$; Then, an observation $x$ is generated from some conditional distribution $p_{\theta}(x | h)$. This defines the generative model:
\begin{displaymath}
p_{\theta}(x,h) = p_{\theta}(h) p_{\theta}(x | h).
\end{displaymath}

The latent code $h$ could be discrete or continuous. The priori $p_{\theta}(h)$ could be any appropriate distribution, with fixed or learnable parameters. For example, the priori could be the standard normal for real-valued $h$ or the factorial Bernoulli for binary vector $h$, as described above.

The conditional distribution $p_{\theta}(x | h)$ could be a multivariate
Gaussian (in case of real-valued data) or Bernoulli (in case of binary data), whose distribution parameters are computed from $h$ with a MLP, similar to those as described in Eq. \ref{eq:SBN-conditional} and \ref{eq:DLGM-conditional} respectively. As an example in this case, we show in the following the details of $p_{\theta}(x | h)$ used in VAEs:
\begin{equation} \label{eq:VAE-conditional}
p_{\theta}(x | h) = \mathcal{N}[x | \mu_{\theta}(h), diag(\sigma^2_{\theta}(h))]
\end{equation}
where $\mu_{\theta}(h)$ is the mean vector, $\sigma^2_{\theta}(h)$ is the vector of standard deviations, and the functions $\mu_{\theta}(h)$ and $\sigma^2_{\theta}(h)$ are represented as MLPs.
It is worthwhile to compare the implementation of the conditional distribution in VAE (Eq. \ref{eq:VAE-conditional}) with that in DGLM (Eq. \ref{eq:DLGM-conditional}). Although both use Gaussians, the former uses full covariances with a simple parameterization, and the latter uses diagonal covariances with a more flexible parameterization by MLPs.

\subsubsection{Prescribed models using auto-regression}
\label{sec:model-def-auto-regression}
Example models include Recurrent Neural Networks (RNNs), LSTM-RNNs, 
Neural Autoregressive Distribution Estimator (NADE) \cite{Larochelle2011TheNA}, Deep AutoRegressive Networks (DARN) \cite{Gregor2014DeepAN}, PixelCNN \cite{van2016conditional} etc.
By directly factorizing the the joint distribution by chain rule, these models do not involve any additional latent random variables, except the multivariate observation variable $x$ itself.

\subsubsection{Implicit models using deep deterministic layers}
\label{sec:model-def-GAN}

Implicit generative models (also called generative neural samplers in \cite{nowozin2016f-gan}) use a latent variable $\epsilon$, sometimes referred to as the input noise variable, and transform
it using a deterministic function $G_\theta(\epsilon)$, usually represented by a MLP, to define:
\begin{equation} \label{eq:GAN}
\left\{
\begin{split}
x &= G_{\theta}(\epsilon), \\
\epsilon &\sim p(\epsilon).
\end{split}
\right.
\end{equation}

The input noise is assumed to obey a simple priori $p(\epsilon)$, e.g. the standard normal $\mathcal{N}[0,I]$ or the uniform distribution.
The transformation implicitly defines a distribution $p_\theta(x)$.
It is known that if the transformation $G_\theta(\epsilon)$ is invertible, the transformed density $p_{\theta}(x)$ has closed form. 
However point-wise likelihood evaluation is generally intractable with the MLP transformation.
Notably, implicit models allow exact sampling and thus expectations w.r.t. $p_{\theta}(x)$ can be efficiently approximated by Monte Carlo averaging.

\subsection{Model Learning}

As shown above, we can define very flexible directed DGMs $p_{\theta}(x)$, by utilizing DNNs to define the conditional densities.
For this reason, it is often assumed in theoretical analysis that $p_{\theta}(x)$ has infinite capacity (sometime called in the non-parametric setting). In practice, the performances of the models largely depend on how they are optimized in model learning.

Many objective functions and learning methods have been proposed for optimizing generative models.
The optimization criterion used has profound effect on the behavior of the optimized model \cite{theis2016a}.
For example, it is known that the KL divergence is asymmetric, and  optimizing which direction of the KL divergence leads to different trade-offs\footnote{The KL approximation covers the data distribution while reverse-KL has more of a mode-seeking behavior \cite{Minka2005DivergenceMA,nowozin2016f-gan}. 
	The fit corresponding to the Jensen-Shannon divergence is somewhere between KL and reverse-KL.
	}.
Different applications require different trade-offs. How to choose the right criterion for a given application remains to be an tough open question, and is usually empirically determined.

Learning under most criteria is provably consistent given infinite model capacity and data.
Nevertheless, among various criteria, maximum likelihood is appealing due to its nice property (consistency, statistical efficiency, and functional invariance).
For many applications, log-likelihood (or equivalently Kullback-Leibler divergence) has been the de-facto standard for training and evaluating generative models. However, the likelihood of many interesting models is computationally intractable.
The motivation for introducing new training criteria and training methods is often the wish to fit probabilistic models with respect to a different objective that is more or less related to log-likelihood.

Recently, there are increased interests in pairing the targe generative model with some auxiliary model, and jointly training the two models. This basic idea has been proposed and enhanced many times - initially by Helmholtz machines \cite{Hinton1995the,dayan1995helmholtz} and recently by NVIL \cite{mnih2014neural}, VAEs \cite{kingma2014auto-encoding}, RWS \cite{bornschein2014reweighted}, IWAE \cite{burda2015importance}, GANs \cite{goodfellow2014generative}, Joint-stochastic-approximation (JSA) \cite{xu2016joint} and so on.

Depending on the objective functions used in joint training of the target model and the auxiliary model, there exist several representative classes of learning algorithms, as summarized in Table \ref{table:directed-learning-summary}.
Suppose that data $\mathcal{D} = \left\lbrace x_1, \cdots, x_n \right\rbrace$, which consists of $n$ IID (independent and identically distributed) observations drawn from the true but unknown data distribution $p_0(x)$ with support $\mathcal{X} $.
$\tilde{p}(x) \triangleq \frac{1}{n} \sum_{k=1}^{n} \delta(x-x_n)$ denotes the empirical distribution.
\textbf{Maximum likelihood learning} is defined by maximizing the data log-likelihood
\begin{equation} \label{eq:MLE}
\max_{\theta} \sum_{k=1}^{n} log p_\theta(x_k),
\end{equation}
which is equivalently to minimizing
the Kullback-Leibler (KL) divergence between the data distribution and the generative model:	
\begin{displaymath}
\min_{\theta} KL\left[ \tilde{p}(x) || p_\theta(x) \right].
\end{displaymath}

\subsubsection{Variational learning} \label{sec:variational-learning}

The first class, which we refer to as variational learning, includes VAE \cite{kingma2014auto-encoding}, stochastic backpropagation \cite{Rezende2014StochasticBA}, NVIL \cite{mnih2014neural}, MuProp \cite{Gu2015MuPropUB} and so on. Variational learning is mainly applied to prescribed models and uses the variational lower bound of the marginal log-likelihood as the single objective function to optimize the target model and auxiliary model.

A prescribed latent variable model could be generally defined as
\begin{displaymath}
p_\theta(x,h) \triangleq p_\theta(h) p_\theta(x|h), 
\end{displaymath}
which consists of observation variable $x$, hidden variables (or say latent code) $h$, and parameters $\theta$.
It is usually intractable to directly evaluate and maximize the marginal log-likelihood $log p_\theta(x)$.
Following the variational inference approach, we introduce an auxiliary inference model $q_\phi(h|x)$ with parameters $\phi$, which serves as an approximation to the exact posterior $p_\theta(h|x)$, and is often called the \emph{variational distribution} in the context of variational inference \cite{zhang2017advances}. 
The implementation of $q_\phi(h|x)$ is technically similar to implementing $p_\theta(x|h)$ as a prescribed model as introduced in \ref{sec:directed-model-def}.

In the classic variational EM learning algorithm \cite{frey2005comparison}, there is one set of variational parameters for each training sample $x$, to approximate the posterior $p_\theta(h|x)$.
The above usage of $q_\phi(h|x)$, employing a single, global set of parameters $\phi$ over the entire training set, represents a new concept introduced in VAEs, called \textbf{amortized inference}.
By employing inference model $q_\phi(h|x)$, VAEs amortize the variational parameters over the entire training set, instead of optimizing for each training sample individually.

After introducing $q_\phi(h|x)$, we have\footnote{We can directly derive the lower bound by applying Jensen Inequality:
\begin{displaymath}
log p_\theta(x) = log \sum_{h} p_\theta(x,h)
= log \sum_{h} q_\phi(h|x) \frac{p_\theta(x,h)}{q_\phi(h|x)}
\geq \sum_{h} q_\phi(h|x) log \frac{p_\theta(x,h)}{q_\phi(h|x)}.
\end{displaymath}	
	}
\begin{equation} \label{eq:VLB1}
\begin{split}
log p_\theta(x) &= E_{q_\phi(h|x)} log \left( \frac{p_\theta(x,h)}{q_\phi(h|x)} \right) 
+ KL\left[ q_\phi(h|x) || p_\theta(h|x) \right]\\
&\geq E_{q_\phi(h|x)} log \left(\frac{p_\theta(x,h)}{q_\phi(h|x)} \right) \triangleq \mathcal{L}(x;\theta,\phi),
\end{split}
\end{equation}
The right side of the inequality $\mathcal{L}(x;\theta,\phi)$ is called the variational lower bound (V-LB) or the evidence lower bound (ELBO)\footnote{It is shown in \cite{Roeder2017StickingTL} that in \textbf{estimating ELBO}, even when analytic forms of the KL divergence or entropy in Eq. \ref{eq:VLB2}, \ref{eq:VLB3} are available, sometimes it is better to use Eq. \ref{eq:VLB1}, because it will have lower variance.
Specifically, when $q_\phi(z|x) = p_\theta(z|x)$, the variance of the Monte Carlo estimator of the ELBO is exactly zero, since the estimator takes a constant value, independent of $h \overset{IID}{\sim} q_\phi(z|x)$.
}, which could also rewritten as:
\begin{align}
\mathcal{L}(x;\theta,\phi) &= 
E_{q_\phi(h|x)} \left[ log p_\theta(x|h) \right] 
- KL\left[ q_\phi(h|x) || p_\theta(h) \right] \label{eq:VLB2}\\
&= 
E_{q_\phi(h|x)} \left[ log p_\theta(x,h) \right] 
+ H\left[ q_\phi(h|x) \right] \label{eq:VLB3}
\end{align}
From the objective function Eq. \ref{eq:VLB2}, it can be seen that variational learning performs something like \emph{auto-encoding}. The first term is the expected negative reconstruction error, while the second term (the KL divergence between the approximate posterior and the prior) acts as a regularizer. In this sense, $p_\theta(x|h)$ and $q_\phi(h|x)$ are also referred to as decoder and encoder.

\textbf{Variational learning} is to maximize the V-LB over the training data (denoted by $\mathcal{L}(\theta,\phi) \triangleq \sum_{k=1}^{n} \mathcal{L}(x_k;\theta,\phi)$) w.r.t. both the generative parameters $\theta$ and variational parameters $\phi$:
\begin{equation} \label{eq:VAE_obj}
\max_{\theta, \phi} \mathcal{L}(\theta,\phi) = \max_{\theta, \phi} E_{\tilde{p}(x) q_\phi(h|x)} log \left( \frac{p_\theta(x,h)}{q_\phi(h|x)} \right)
\end{equation}
\emph{In this manner, we actually maximize a lower bound to the true maximum-likelihood objective Eq. \ref{eq:MLE}, hoping to maximizing the objective itself.}

By setting to zeros the gradients of the objective w.r.t. $(\theta,\phi)$, the above optimization problem Eq.\ref{eq:VAE_obj} can be solved by finding the root for the following system of simultaneous equations\footnote{A quick proof for the second equation. Taking the derivative of $\mathcal{L}(\theta,\phi)$ w.r.t. $\phi$, we have
\begin{displaymath}
\nabla_\phi \mathcal{L}(\theta,\phi) = E_{\tilde{p}(x) q_\phi(h|x)}\left[ \nabla_\phi logq_\phi(h|x) \right] 
+ E_{\tilde{p}(x) q_\phi(h|x)}\left[ 
log \left(\frac{p_\theta(x,h)}{q_\phi(h|x)} \right) \times
\nabla_\phi logq_\phi(h|x) \right],
\end{displaymath}
where the first term is equal to zero by Fisher Equality, and the the second term is derived using the so-called \textbf{REINFORCE trick} \cite{Williams1992SimpleSG}.}:
\begin{equation} \label{eq:VAE_gradient}
\left\{
\begin{split}
& E_{\tilde{p}(x) q_\phi(h|x)}\left[ \nabla_\theta logp_\theta(x,h) \right] = 0 \\
& E_{\tilde{p}(x) q_\phi(h|x)}\left[ 
log \left(\frac{p_\theta(x,h)}{q_\phi(h|x)} \right) \times
\nabla_\phi logq_\phi(h|x) \right]
= 0
\end{split}
\right.
\end{equation}

It can be shown that Eq.(\ref{eq:VAE_gradient}) exactly follows the form of Eq.(\ref{eq:SA}), so that we can apply the SA algorithm, as shown in Algorithm \ref{alg:VAE}, to find its root and thus solve the optimization problem Eq.(\ref{eq:VAE_obj}).

The above gives the basics of variational learning.
In the following, we present existing problems, recent advances and open questions related to variational learning.

\begin{algorithm}[tb]
	\caption{\textbf{Variational learning}, represented as the SA algorithm with multiple moves}
	\label{alg:VAE}
	\begin{algorithmic}
		\REPEAT
		\STATE \underline{Monte Carlo sampling:}
		Draw a minibatch $\mathcal{M} \sim \tilde{p}(x) q_\phi(h|x)$;
		\STATE \underline{SA updating:}
		Update $\theta$ by ascending: 
		$\frac{1}{|\mathcal{M}|} \sum_{(x,h) \sim \mathcal{M}}
		\nabla_\theta logp_\theta(x,h)$;\\
		Update $\phi$ by ascending: 
		$\frac{1}{|\mathcal{M}|} \sum_{(x,h) \sim \mathcal{M}}
		log \left(\frac{p_\theta(x,h)}{q_\phi(h|x)} \right) \times
		\nabla_\phi logq_\phi(h|x)$;
		\UNTIL{convergence}
	\end{algorithmic}
\end{algorithm}

\begin{enumerate}

\item
\textbf{Handling of continous latent variables, via the reparameterization trick.}
	
A major trouble with variational learning is that while the gradient w.r.t. $\theta$ is usually well-behaved, the usual Monte Carlo gradient estimator w.r.t. $\phi$ is known to have high variance.
In the basic form of variational learning as shown in Algorithm \ref{alg:VAE}, the update of $\phi$ is a REINFORCE-like update rule which trains slowly because it does not use the log-likelihood gradients with respect to latent variables (i.e. $\nabla_h log p_\theta(x,h)$) \cite{burda2015importance}.
To address this problem, there are a lot of efforts.

In the case of using continuous hidden variables $h$, the reparameterization trick \cite{Kingma2014SemiSupervisedLW,Rezende2014StochasticBA} has been developed, which essentially implements gradient back-propagation through continuous random variables and provides an effective method for handling continuous hidden variables.

Suppose that the inference model $q_\phi(h|x)$ can be reparameterized by using a differentiable, deterministic transformation $h_\phi(x,\epsilon)$ of a noise variable $\epsilon$:
\begin{equation} \label{eq:repara}
	h = h_\phi(x,\epsilon)~\text{with}~\epsilon \sim \mathcal{N}[0,I].
\end{equation}
Then, we have
\begin{equation} \label{eq:repara-grad}
\nabla_\phi E_{q_\phi(h|x)} \left[  log \left( \frac{p_\theta(x,h)}{q_\phi(h|x)} \right) \right] =
E_{\epsilon \sim \mathcal{N}[0,I]} \left[  \nabla_\phi log \left( \frac{p_\theta(x, h_\phi(x,\epsilon))}{q_\phi(h_\phi(x,\epsilon)|x)} \right) \right],	
\end{equation}
from which the gradient can be computed by a direct application of the chain rule.
Intuitively, the reparameterization trick provides more informative
$\phi$-gradients by exposing the dependence of sampled
latent variables $h$ on variational parameters $\phi$ through involving $\nabla_h log p_\theta(x,h)$. In
contrast, the REINFORCE $\phi$-gradient estimate only depends
on involving $\nabla_\phi logq_\phi(h|x)$.

\textbf{Improving $\phi$-gradient estimator beyond simple reparameterization  \cite{Roeder2017StickingTL}.}
Under the reparameterization of $h$ as shown in Eq. \ref{eq:repara}, we can decompose the total derivative (TD) of the integrand of estimator Eq. \ref{eq:repara-grad} w.r.t. $\phi$ as follow:
\begin{displaymath}
\begin{split}
&\underbrace{\nabla_\phi \left[ log  p_\theta(x, h_\phi(x,\epsilon)) - log q_\phi(h_\phi(x,\epsilon)|x) \right]}_\text{total derivative} \\
=&\nabla_\phi \left[ log p_\theta(h_\phi(x,\epsilon)|x) + log p_\theta(x) - log q_\phi(h_\phi(x,\epsilon)|x) \right]\\
=&\underbrace{\left\lbrace \nabla_h \left[ log p_\theta(h|x) - log q_\phi(h|x)\right]\right\rbrace |_{h=h_\phi(x,\epsilon)} \nabla_\phi h_\phi(x,\epsilon)}_\text{path derivative} - \underbrace{\left\lbrace \nabla_\phi log q_\phi(h|x) \right\rbrace |_{h=h_\phi(x,\epsilon)}}_\text{score function}
\end{split}
\end{displaymath}
By the chain rule, the reparameterized gradient estimator w.r.t. $\phi$ decomposes into two parts, which are called the path derivative and score function components respectively.
The path derivative measures dependence on $\phi$ through the sample $h$. 
The score function measures dependence on $\phi$ through $log q_\phi(h|x)$, without considering how the sample $h$ changes as a function of $\phi$.

Notably, the score function term has expectation zero. If we simply remove that term, we maintain an unbiased estimator of the true gradient by using only the first term, which is called the path derivative (PD) gradient estimator.
The path derivative estimator has the desirable property that as $q_\phi(z|x)$ approaches $p_\theta(z|x)$, the variance of this estimator goes to zero.
This modification can be interpreted as adding the score function as a control variate.

\item
\textbf{Handling of discrete latent variables.}

For variation learning to work with discrete latent variables, such as Bernoulli or multinomial, there exists the difficulty of back-propagation of gradients through discrete random variables.
There have some efforts to address this difficulty.
\begin{itemize}
\item
The REINFORCE-like trick is employed with various variance-reduction techniques in \cite{mnih2014neural} (NVIL - centering the learning signal $log p_\theta(x,h) - log q_\phi(h|x)$, introducing an input-dependent baseline via a neural network, and doing variance normalization via dividing the centered learning signal by a running estimate of its standard deviation), \cite{mnih2016variational} (VIMCO - Variational Inference for Monte Carlo Objectives, which basically is NVIL plus using the tighter bound IW-LB Eq. \ref{eq:IW-LB}). 
\item
VQ-VAE \cite {vqvae} approximates the gradient with the straight-through (ST) estimator \cite{Bengio2013EstimatingOP}. 
\item
A few studies utilize the Concrete \cite{maddison2016concrete} or Gumbel-softmax \cite{jang2016categorical} distribution as a continuous approximation to a multinomial distribution, and then use 'reparameterization trick', although the gradients of these relaxations are biased.
\end{itemize}

\item
\textbf{Handling of discrete observations.}

The application of variational learning to discrete observations has no theoretical difficulty for gradient propagation, but has been far less successful \cite{Bowman2016GeneratingSF,Miao2016NeuralVI}.
It is found that the LSTM decoder in textual VAE does not make effective use of the latent code during training. This training collapse problem may reflect structure mismatch between the encoder and decoder, and is alleviated in \cite{vaeccnn} after controlling the contextual capacity of the decoder by using dilated CNN.
Structure mismatch between the encoder and decoder causes an irreducible biased gap between the marginal log-likelihood and the ELBO for VAEs.

\item
\textbf{More expressive inference model, going beyond factorized-Gaussian.}

Generally, the quality of variational learning crucially relies on the expressiveness of the inference model to capture the
true posterior distribution. 
It is shown in \cite{Kingma2016ImprovingVI} that using more expressive model classes can lead to substantially better results, both visually
and in terms of log-likelihood bounds. The work in \cite{Chen2016VariationalLA} also suggests that highly expressive inference models are essential in presence of a strong decoder to allow the model to make use of the latent space at all.
There are a number of strategies for increasing the expressiveness of the inference model.
\begin{itemize}
\item
Normalization flow (\citet{rezende2015variational}, NICE \citet{dinh2014nice}, real-valued non-volume preserving (real NVP) \citet{dinh2016density}), Inverse autoregressive flow (IAF) \citet{Kingma2016ImprovingVI}. The basic idea is to use a change of variables procedure for constructing complex distributions by transforming probability densities through a series of invertible mappings.
\item 
Deep generative models can be extended with \emph{auxiliary variables}
which leave the generative model unchanged but make
the variational distribution more expressive.
This idea has been employed in works such as auxiliary deep generative models (ADGM, \citet{maaloe2016auxiliary}), hierarchical variational models (HVM, \citet{ranganath2016hierarchical}) and Hamiltonian
variational inference (HVI, \citet{salimans2015markov}).
\item
Instead of using the prescribed $q_\phi(x|h)$ (often a diagonal-covariance Gaussian parameterized by a MLP), some works propose to use an implicit model (sometimes called a black-box procedure \citet{mescheder2017adversarial}) to implement $q_\phi(h|x)$. Note that $q_\phi(h|x)$ is likelihood-free in this manner.
In order to maximize the V-LB which is rewritten as in the following\footnote{In the following AVB derivation, $p(h)$ is used, dropping the dependence on $\theta$.}
\begin{equation} \label{eq:AVB-VLB}
\max_{\theta, \phi} \mathcal{L}(\theta,\phi) = \max_{\theta, \phi} E_{\tilde{p}(x) q_\phi(h|x)} \left[ log p(h) - log q_\phi(h|x) + log p_\theta(x|h) \right],
\end{equation}
it is proposed in AVB (Adversiarial Variational Bayes, \cite{mescheder2017adversarial}) to implicitly represent the term
\begin{displaymath}
log q_\phi(h|x) - log p(h),
\end{displaymath}
as the optimal value of an additional, infinite capacity, real-valued discriminator $V_\psi(x,h)$, given by
\begin{equation} \label{eq:AVB-v}
\max_{\psi} E_{\tilde{p}(x) q_\phi(h|x)} \left[log \sigma\left(V_\psi(x,z)\right) \right]
+ E_{\tilde{p}(x) p(h)} \left[ log \left(1-\sigma\left(V_\psi(x,z)\right)\right) \right].
\end{equation}
Then we can formulate the variational learning with black-box inference model as jointly optimizing
\footnote{
Through plugging, we transform the optimization Eq. \ref{eq:AVB-VLB} into Eq. \ref{eq:AVB}. 
Plugging itself only guarantees the equality of the objectives, namely $ \mathcal{L}^{AVB}_\psi(\theta,\phi) |_{\psi=\psi^*_\phi} = \mathcal{L}(\theta,\phi)$ for the optimal $\psi^*_\phi$, determined by Eq. \ref{eq:AVB-v}.
\textbf{For the plugging trick to be valid}, we need to additionally check the gradients
\begin{displaymath}
\left\lbrace \nabla_{\theta,\phi} \mathcal{L}^{AVB}_\psi(\theta,\phi) \right\rbrace |_{\psi=\psi^*_\phi} = \nabla_{\theta,\phi} \mathcal{L}(\theta,\phi),
\end{displaymath}
because in optimizing the first line in Eq. \ref{eq:AVB}, we ignore the implicit dependence of $V_\psi(x,h)$ on $\phi$, i.e. we take the gradient of $\mathcal{L}^{AVB}_\psi(\theta,\phi)$ w.r.t. $\left( \theta,\phi \right)$ while treating $\psi$ as constant.
This equality clearly holds for the gradient w.r.t. $\theta$.
For checking the gradient w.r.t. $\phi$, we have
$\nabla_\phi \mathcal{L}^{AVB}_\psi(\theta,\phi) = 
E_{q_\phi(h|x)} \left[ \nabla_\phi log q_\phi(h|x) \left(  - V_\psi(x,h) + log p_\theta(x|h) \right)  \right]$, and thus
\begin{displaymath}
\left\lbrace \nabla_\phi \mathcal{L}^{AVB}_\psi(\theta,\phi) \right\rbrace |_{\psi=\psi^*_\phi} 
=E_{\tilde{p}(x) q_\phi(h|x)} \left[ \nabla_\phi log q_\phi(h|x) \left( -log q_\phi(h|x) + log p(h) + log p_\theta(x|h) \right)  \right],
\end{displaymath}
which is indeed equal to $\nabla_{\phi} \mathcal{L}(\theta,\phi)$.
This check is equivalent to check that $E_{q_\phi(h|x)} \left[ \nabla_\phi \left\lbrace V_\psi(x,h) |_{\psi=\psi^*_\phi} \right\rbrace  \right] = 0 $, which is taken in the Proposition 2 in the original AVB paper \citep{mescheder2017adversarial}. See Appendix \ref{sec:appendix-plugging} for more discussion.
}
\begin{equation}
\label{eq:AVB}
\left\{
\begin{split}
& \max_{\theta, \phi} E_{\tilde{p}(x) q_\phi(h|x)} \left[ - V_\psi(x,h) + log p_\theta(x|h) \right] \triangleq \mathcal{L}^{AVB}_\psi(\theta,\phi)  \\
& \max_{\psi} E_{\tilde{p}(x) q_\phi(h|x)} \left[log \sigma\left(V_\psi(x,z)\right) \right]
+ E_{\tilde{p}(x) p(h)} \left[ log \left(1-\sigma\left(V_\psi(x,z)\right)\right) \right] \\
\end{split}
\right.
\end{equation}
By setting to zeros the gradients of the objective w.r.t. $(\theta,\phi,\psi)$, the above optimization problem can be solved by applying the SA algorithm.
In this joint optimization, AVB actually combines variational learning with adversarial learning (hence called ``AVB'').

The price for using black-box inference model is that we need to jointly optimize three set of parameters $(\theta,\phi,\psi)$.
We trade the increasing expressiveness of the inference model with the increasing difficulty of optimization, which reminds us the No Free Lunch Theorems for optimization \cite{Wolpert1997NoFL}.
It is important to understand the trade-offs between different learning methods, so that we can choose appropriate methods for different applications.
\end{itemize}

\item
\textbf{Reducing the amortization gap.}

From Eq. \ref{eq:VLB1}, we see that the ELBO is tight when $q_\phi(h|x) = p_\theta(h|x)$. 
The above studies in finding more expressive inference model are generally motivated to reduce the gap between the marginal log-likelihood and the ELBO in VAEs. It can be seen from the following discussion, these studies mainly reduce the approximation gap.

Recently, there are careful studies in understanding what factors cause this gap between the marginal log-likelihood and the ELBO in VAEs, which is called the inference gap in \cite{cremer2018inference}.
Let $q_{\phi^*_x}(h|x)$ refers to the optimal approximation within the family $\mathcal{Q}$, i.e. $q_{\phi^*_x}(h|x) = argmin_{q \in \mathcal{Q}} KL\left[ q_\phi(h|x) || p_\theta(h|x) \right] $.
For each training sample $x$, we optimize to find the sample-specific optimal variational parameter $\phi^*_x$.
The inference gap decomposes as the sum of approximation and amortization gaps:
\begin{displaymath}
\begin{split}
\underbrace{log p_\theta(x) - \mathcal{L}(\theta,\phi)}_\text{Inference Gap} &= \underbrace{log p_\theta(x) - \mathcal{L}(\theta,{\phi^*_x})}_\text{Approximation Gap} + \underbrace{\mathcal{L}(\theta,{\phi^*_x}) - \mathcal{L}(\theta,{\phi})}_\text{Amortization Gap}\\
&= \underbrace{KL\left[ q_{\phi^*_x}(h|x) || p_\theta(h|x) \right]}_\text{Approximation Gap} + \underbrace{KL\left[ q_\phi(h|x) || p_\theta(h|x) \right]- KL\left[ q_{\phi^*_x}(h|x) || p_\theta(h|x) \right]}_\text{Amortization Gap}
\end{split}
\end{displaymath}
It is found in \cite{cremer2018inference} that divergence from the true posterior is often due to amortized inference models,
rather than the limited complexity of the approximating
distribution family; this is due partly to the generator learning to accommodate the choice of approximation.
Regarding this, a hybrid approach is proposed in \cite{kim2018semi}, which uses Amortized variational inference (AVI) to initialize the variational parameters and runs stochastic variational inference (SVI) to refine them. 
Crucially, the local SVI procedure is itself differentiable, so the inference network and generative model can be trained end-to-end with gradient-based optimization.

\item
\textbf{Exploring alternative training criteria, e.g. learning with tighter lower bounds.}

Recently, a tighter log-likelihood lower bound than the variational lower bound is derived from using importance weighting and applied to improve the variational learning.
The resulting new learning algorithm is called \textbf{importance weighted autoencoders (IWAEs)} \cite{burda2015importance}.

Using importance weighting with $q_\phi(h|x)$ as the proposal, we can approximate the marginal likelihood $p_\theta(x)$ as an average of importance weights:
\begin{equation} \label{eq:IW-lik}
\begin{split}
p_\theta(x) = \sum_{h} q_\phi(h|x) \frac{p_\theta(x,h)}{q_\phi(h|x)}
\approx \frac{1}{K} \sum_{k=1}^{K} \frac{p_\theta(x,h_k)}{q_\phi(h_k|x)}, h_k \sim q_\phi(h_k|x),
\end{split}
\end{equation}
Let 
\begin{equation} \label{eq:IS-weights}
w_{\theta,\phi}^k = \frac{p_\theta(x,h_k)}{q_\phi(h_k|x)}, ~~\bar{w}_{\theta,\phi}^k = \frac{w_{\theta,\phi}^k}{\sum_{k'=1}^{K} w_{\theta,\phi}^{k'}}
\end{equation}
be the importance weights and the normalized weights respectively.
It can be seen that the average of importance weights,
\begin{displaymath}
\hat{p}^K_{\theta, \phi}(x) = \frac{1}{K} \sum_{k=1}^{K} w_{\theta,\phi}^k,
\end{displaymath} 
is an \emph{unbiased} estimator of the marginal likelihood $p_\theta(x)$:
\begin{displaymath}
E_{q_\phi(h_{1:K}|x)}\left[ \hat{p}^K_{\theta, \phi}(x) \right] = p_\theta(x),
\end{displaymath}
where $q_\phi(h_{1:K}|x) \triangleq \prod_{k=1}^{K} q_\phi(h_k|x)$. By Jensen’s Inequality, we obtain an importance-weighting lower bound (IW-LB) on the marginal log-likelihood $log p_\theta(x)$:
\begin{equation} \label{eq:IW-LB}
\mathcal{I}_K(x;\theta,\phi) \triangleq
E_{q_\phi(h_{1:K}|x)} \left[ log \hat{p}^K_{\theta, \phi}(x) \right] \le 
log E_{q_\phi(h_{1:K}|x)} \left[ \hat{p}^K_{\theta, \phi}(x) \right] = log p_\theta(x),
\end{equation}
and further we have \cite{burda2015importance}
\begin{displaymath}
\mathcal{L}(x;\theta,\phi) \le
\mathcal{I}_K(x;\theta,\phi) \le 
\mathcal{I}_{K+1}(x;\theta,\phi) \le
log p_\theta(x),
\end{displaymath}
and $\mathcal{I}_K$ approaches $log p_\theta(x)$ as $K$ goes to infinity. Note that the IW-LB in the special case of $K=1$ is equal to the standard V-LB: $\mathcal{I}_1(x;\theta,\phi) = \mathcal{L}(x;\theta,\phi)$. Therefore, IW-LB (when $K>1$) is a tighter lower bound on the marginal log-likelihood\footnote{It is shown in \cite{le2017auto} that the gap between the IW-LB and the marginal log-likelihood is a KL divergence on an extended sampling space:
\begin{displaymath}
\mathcal{I}_K(x;\theta,\phi) = log p_\theta(x)
-E_{q_\phi(h_{1:K}|x)} \left[ log \left(
\frac{q_\phi(h_{1:K}|x)}{\frac{1}{K} \sum_{k=1}^{K} q_\phi(h_{1:K}|x) \frac{p_\theta(h_k|x)}{q_\phi(h_k|x)}} \right)\right]
\end{displaymath}
}. 

IWAE is to maximize the IW-LB over the training data (denoted by $\mathcal{I}_K(\theta,\phi) \triangleq \sum_{k=1}^{n} \mathcal{I}_K(x_k;\theta,\phi)$) w.r.t. both $\theta$ and $\phi$.
As with the VAE, we need to use the reparameterization trick to derive a low-variance update rule. Thus in the original paper \cite{burda2015importance}, IWAE is only applied to DGMs consisting of continuous latent variables.\footnote{Later in VIMCO \cite{mnih2016variational}, the IW-LB is also applied to handle discrete latent variables, with additional variance reduction techniques.}

Specifically, we introduce a set of parameter-free, auxiliary noise variables $\epsilon_1,\cdots,\epsilon_K$. Then, the expectation under $q_\phi(h_{1:K}|x)$ in $\mathcal{I}_K(x;\theta,\phi)$ can be calculated as:
\begin{displaymath}
\begin{split}
\mathcal{I}_K(x;\theta,\phi) &\triangleq
E_{q_\phi(h_{1:K}|x)} \left[ log \left(\frac{1}{K} \sum_{k=1}^{K} \frac{p_\theta(x,h_k)}{q_\phi(h_k|x)} \right)\right] \\
&= E_{\epsilon_{1:K}} \left[ log \frac{1}{K} \left( \sum_{k=1}^{K} \frac{p_\theta(x,h_\phi(x,\epsilon_k))}{q_\phi(h_\phi(x,\epsilon_k)|x)} \right) \right],
\end{split}
\end{displaymath}
where $h_\phi(x,\epsilon_k)$ is the deterministic transformation of  $\epsilon_k$ such that the transformed density follows $q_\phi(h_k|x)$.
The derivative of $\mathcal{I}_K(\theta,\phi)$ w.r.t. $(\theta,\phi)$ can be calculated as:
\begin{displaymath}
\nabla_{(\theta,\phi)} \mathcal{I}_K(x;\theta,\phi) = 
E_{\epsilon_{1:K}} \left[ \nabla_{(\theta,\phi)} log \left(\frac{1}{K} \sum_{k=1}^{K} \frac{p_\theta(x,h_\phi(x,\epsilon_k))}{q_\phi(h_\phi(x,\epsilon_k)|x)} \right) \right].
\end{displaymath}
It can be seen that in the special case of $K=1$, the IWAE update rule reduces to the VAE update rule.

Some more notes. 
\begin{itemize}
\item
Often, the second term (the KL-divergence) in the VL-B Eq. \ref{eq:VLB2} can be integrated analytically, such that only the expected reconstruction error (the first term in the VL-B Eq. \ref{eq:VLB2}) requires estimation by sampling.
Unfortunately, no analogous trick of reducing variances applies for IWAE when $K>1$. In principle, the IWAE updates may be higher variance for this reason. However, it is reported in their experiments \cite{burda2015importance} that the performance
of the two update rules was indistinguishable in the case of $K=1$.
\item
Recently, it is reported in \cite{rainforth2018tighter} that increasing the number of importance sampling particles, $K$, to tighten the bound IW-LB, degrades the signal-to-noise ratio (SNR) of the gradient estimates for the inference network, inevitably deteriorating the overall learning process.
In short, this behavior manifests because even though
increasing $K$ decreases \emph{the standard deviation} of the gradient
estimates w.r.t. $\phi$ (Eq. \ref{eq:RWS_phi_gradient_mc}), it decreases \emph{the magnitude} of the true gradient faster, such that the relative variance increases.
Their results suggest that it may be best to use distinct objectives for learning the generative and inference networks.

\end{itemize}

\end{enumerate}

\subsubsection{Wake-sleep learning}

The wake-sleep (WS) learning algorithm is originally proposed for training Helmholtz machines\footnote{In the literature, a Helmholtz machine (HM) is often referred to as a combination of the model itself and the WS algorithm. But we describe the WS algorithm in the general form, which is not limited to train a SBN as presented in the original work \cite{Hinton1995the,dayan1995helmholtz}.}.
It represents a class of algorithms to train a prescribed latent variable model $p_\theta(x,h)$ together with an inference model $q_\phi(h|x)$, which uses two objective functions - the variational lower bound (V-LB) for optimizing $\theta$ over empirical samples $\tilde{p}(x)$ (\textbf{wake phase $\theta$ update}) and the inclusive KL-divergence
$KL[p_\theta(h|x)|| q_\phi(h|x)]$
for optimizing $\phi$ over model samples $p_\theta(x)$ (\textbf{sleep phase $\phi$ update}). 

The motivation is mixed. Learning through many cycles of sensing and dreaming seems to be biological attractive.
In addition, note that in variational learning, the optimization w.r.t. $\theta$ and $\phi$ could be rewritten as:
\begin{equation} \label{eq:VAE-obj-re}
\left\{
\begin{split}
&\max_\theta \mathcal{L}(\theta,\phi) \Leftrightarrow \min_\theta KL\left[ \tilde{p}(x) q_\phi(h|x) || p_\theta(x,h) \right] \\
&\max_\phi \mathcal{L}(\theta,\phi) \Leftrightarrow \min_\phi E_{\tilde{p}(x)} KL\left[ q_\phi(h|x) || p_\theta(h|x) \right]
\end{split}
\right.
\end{equation}
The optimization of the variational lower bound $\mathcal{L}(\theta,\phi)$ with respect to $\phi$ amounts to minimize the exclusive KL-divergence $KL(q_\theta(h|x|| p_\phi(h|x)))$ over empirical samples, which has the undesirable effect of high variance mentioned before.
Updating $\phi$ by optimizing the inclusive KL-divergence would be easier than by optimizing the exclusive KL-divergence.
With this modification of updating $\phi$ while performing the same update of $\theta$ as in variational learning, we obtain \textbf{the WS algorithm}:
\begin{equation} \label{eq:WS-obj}
\left\{
\begin{split}
&\max_\theta \mathcal{L}(\theta,\phi) \Leftrightarrow \min_\theta KL\left[ \tilde{p}(x) q_\phi(h|x) || p_\theta(x,h) \right]\\
&\min_\phi E_{p_\theta(x)} KL\left[ p_\theta(h|x) || q_\phi(h|x) \right] \Leftrightarrow \min_\phi KL\left[ p_\theta(x,h) || p_\theta(x) q_\phi(h|x) \right]
\end{split}
\right.
\end{equation}
Note that the WS algorithm not only changes the direction of the KL-divergence in optimizing $\phi$, but also changes the samples over which $\phi$ is updated - from using empirical samples to using model samples. Hence, updating $\phi$ is called sleep phase update in the WS algorithm.

By setting to zeros the gradients of the objective w.r.t. $(\theta,\phi)$, the above optimization problem can be solved by finding the root for the following system of simultaneous equations:
\begin{equation} \label{eq:WS_gradient}
\left\{
\begin{split}
& E_{\tilde{p}(x) q_\phi(h|x)}\left[ \nabla_\theta logp_\theta(x,h) \right] = 0 \\
& E_{p_\theta(x,h)}\left[ \nabla_\phi logq_\phi(h|x) \right]
= 0
\end{split}
\right.
\end{equation}

It can be shown that Eq.(\ref{eq:WS_gradient}) exactly follows the form of Eq.(\ref{eq:SA}), so that we can apply the SA algorithm, as shown in Algorithm \ref{alg:WS}, to find its root and thus solve the optimization problem Eq.(\ref{eq:WS-obj}).

\begin{prop}
	If Eq.(\ref{eq:WS_gradient}) is solvable, then we can apply the SA algorithm to find its root.
\end{prop}

\begin{proof}
	This can be readily shown by recasting Eq.(\ref{eq:WS_gradient}) in the form of $f(\lambda) = 0$, with $\lambda \triangleq (\theta, \phi)^T$, $z \triangleq (x,h,x',h')^T$, $p(z; \lambda) \triangleq \tilde{p}(x) q_\phi(h|x) \times p_\theta(x',h')$, and 
	\begin{displaymath}
	F(z; \lambda) \triangleq \left( \begin{array}{c}
	\nabla_\theta logp_\theta(x,h) \\
	\nabla_\phi logq_\phi(h'|x')
	\end{array} \right).
	\end{displaymath}
\end{proof}

\begin{algorithm}[tb]
	\caption{\textbf{Wake-sleep learning}, represented as the SA algorithm with multiple moves}
	\label{alg:WS}
	\begin{algorithmic}
		\REPEAT
		\STATE \underline{Monte Carlo sampling:}
		Draw a wake minibatch $\mathcal{M} \sim \tilde{p}(x) q_\phi(h|x)$;\\
		Draw a sleep minibatch $\mathcal{M}' \sim p_\theta(x',h')$;
		\STATE \underline{SA updating:}
		Wake-phase $\theta$ update by ascending: 
		$\frac{1}{|\mathcal{M}|} \sum_{(x,h) \sim \mathcal{M}}
		\nabla_\theta logp_\theta(x,h)$;\\
		Sleep-phase $\phi$ update by ascending: 
		$\frac{1}{|\mathcal{M}'|} \sum_{(x,h) \sim \mathcal{M}'}
		\nabla_\phi logq_\phi(h'|x')$;
		\UNTIL{convergence}
	\end{algorithmic}
\end{algorithm}

Recently, importance weighting is applied to improve the WS learning.
The resulting new learning algorithm is called \textbf{reweighted wake-sleep (RWS)} \cite{bornschein2014reweighted}.
A historic note: the importance weighting technique is first used in RWS to estimate the gradients of the marginal log-likelihood, and later used in IWAE to  estimate the marginal log-likelihood and derive a tighter log-likelihood lower bound than the variational lower bound as introduced before in Eq. \ref{eq:IW-lik}.

\textbf{Wake-phase $\theta$ update in RWS.} There are two ways to derive the gradient of the marginal log-likelihood based on importance weighting.
First, it is shown in RWS \cite{bornschein2014reweighted} that
\begin{equation} \label{eq:RWS_gradient_mc}
\begin{split}
\nabla_\theta log p_\theta(x)
&= \frac{1}{p_\theta(x)} \sum_{h} \nabla_\theta p_\theta(x,h)
= \frac{1}{p_\theta(x)} \sum_{h} q_\phi(h|x) \frac{p_\theta(x,h)}{q_\phi(h|x)} \nabla_\theta log p_\theta(x,h)\\
&\approx \frac{1}{p_\theta(x)} \frac{1}{K} \sum_{k=1}^{K} \frac{p_\theta(x,h_k)}{q_\phi(h_k|x)} \nabla_\theta log p_\theta(x,h_k), h_k \sim q_\phi(h_k|x) \\
&= \sum_{k=1}^{K} \bar{w}_{\theta,\phi}^k \nabla_\theta log p_\theta(x,h_k), \text{(by Eq. \ref{eq:IW-lik})}
\end{split}
\end{equation}
where, as defined in Eq. \ref{eq:IS-weights} but rewritten here, $w_{\theta,\phi}^k \triangleq \frac{p_\theta(x,h_k)}{q_\phi(h_k|x)}$ is the importance weights,
$\bar{w}_{\theta,\phi}^k = \frac{w_{\theta,\phi}^k}{\sum_{k'=1}^{K} w_{\theta,\phi}^{k'}}$ is the normalized weights.
It is also noted in \cite{bornschein2014reweighted} that the above is a \emph{biased} estimator for the gradient of the marginal log-likelihood w.r.t $\theta$, because it implicitly contains a division by the estimated $p_\theta(x)$ from Eq. \ref{eq:IW-lik}.

Second, we show that the above gradient estimator (Eq. \ref{eq:RWS_gradient_mc}) is in fact for the gradient of the IW-LB  of the marginal log-likelihood (as defined in Eq. \ref{eq:IW-LB}), not for the gradient of the marginal log-likelihood itself.
This can be readily seen from calculating
\begin{equation} \label{eq:RWS_gradient}
\nabla_\theta \mathcal{I}_K(x;\theta,\phi) =
E_{q_\phi(h_{1:K}|x)} \left[ \nabla_\theta log \left(\frac{1}{K} \sum_{k=1}^{K} \frac{p_\theta(x,h_k)}{q_\phi(h_k|x)} \right)\right],
\end{equation}
whose Monte Carlo estimator is exactly Eq. \ref{eq:RWS_gradient_mc}.

Note that the gradient w.r.t. $\theta$ in variational learning (Eq. \ref{eq:VAE_gradient})
and in WS learning are the same (Eq. \ref{eq:WS_gradient}), and both are equivalent to the above gradient w.r.t. in $\theta$ Eq. \ref{eq:RWS_gradient} when using only $K=1$ sample.

\textbf{Wake phase $\phi$ update in RWS.} Analogous to Eq. \ref{eq:RWS_gradient_mc} (wake-phase $\theta$ update), importance sampling is used in RWS to derive gradients for the wake phase $\phi$ update as well, that is :
\begin{displaymath}
\min_\phi E_{\tilde{p}(x)} KL\left[ p_\theta(h|x) || q_\phi(h|x) \right]
\end{displaymath}
The gradient w.r.t. $\phi$ involves:
\begin{equation} \label{eq:RWS_phi_gradient_mc}
\begin{split}
\sum_{h} p_\theta(h|x) \nabla_\phi log q_\phi(h|x)
&= \frac{1}{p_\theta(x)} \sum_{h} q_\phi(h|x) \frac{p_\theta(x,h)}{q_\phi(h|x)} \nabla_\phi log q_\phi(h|x)\\
&\approx \frac{1}{p_\theta(x)} \frac{1}{K} \sum_{k=1}^{K} \frac{p_\theta(x,h_k)}{q_\phi(h_k|x)} \nabla_\phi log q_\phi(h_k|x), h_k \sim q_\phi(h_k|x) \\
&= \sum_{k=1}^{K} \bar{w}_{\theta,\phi}^k \nabla_\phi log q_\phi(h_k|x), \text{(by Eq. \ref{eq:IW-lik})}
\end{split}
\end{equation}
with the same importance weights $\bar{w}_{\theta,\phi}^k$ as in Eq. \ref{eq:RWS_gradient_mc}.
Again, the above is a \emph{biased} estimator.
Also note that in the above, it is tricky to apply importance sampling after taking the derivative.
Taking the derivative after applying importance sampling would make things more complicated.

\subsubsection{Joint-stochastic-approximation learning}

The joint stochastic approximation (JSA) learning algorithm is originally proposed in \cite{xu2016joint} for learning a broad class of directed generative models.

Consider a generative model $p_\theta(x,h) \triangleq p_\theta(h) p_\theta(x|h)$, consisting of observation variable $x$, hidden variables (or say latent code) $h$, and parameters $\theta$.
It is usually intractable to directly evaluate and maximize the marginal log-likelihood $log p_\theta(x)$, but it is well-known that we have, according to Fisher Equality,
\begin{displaymath}
\nabla_\theta log p_\theta(x) = E_{p_\theta(h|x)}\left[ \nabla_\theta logp_\theta(x,h)\right] 
\end{displaymath}

Like in the variational learning and the WS learning as reviewed above, JSA also jointly train the target generative model $p_\theta(x,h)$ together with an auxiliary inference model $q_\phi(h|x)$.
A distinctive key idea of JSA learning is that 
in addition to maximizing w.r.t. $\theta$ the marginal log-likelihood, it simultaneously minimizes w.r.t. $\phi$ the \emph{inclusive} KL divergence $KL(p_\theta(h|x)||q_\phi(h|x))$ between the posteriori and the inference model, and fortunately, we can use the SA framework to solve the optimization problem.

\textbf{JSA learning} is to formulate the maximum likelihood learning as jointly optimizing
\begin{equation}
\label{eq:JSA_unsup_obj}
\left\{
\begin{split}
& \min_{\theta} KL\left[ \tilde{p}(x) || p_\theta(x) \right] \\
& \min_{\phi} KL\left[ \tilde{p}(x) p_\theta(h|x)|| \tilde{p}(x) q_\phi(h|x) \right] \\
\end{split}
\right.
\end{equation}
By setting the gradients to zeros, the above optimization problem can be solved by finding the root for the following system of simultaneous equations:
\begin{equation}
\label{eq:JSA_unsup_gradient}
\left\{
\begin{split}
& E_{\tilde{p}(x) p_\theta(h|x)}\left[\nabla_\theta logp_\theta(x,h)\right] = 0 \\
& E_{\tilde{p}(x) p_\theta(h|x)}\left[ \nabla_\phi logq_\phi(h|x)\right] = 0 \\
\end{split}
\right.
\end{equation}

It can be shown that Eq.(\ref{eq:JSA_unsup_gradient}) exactly follows the form of Eq.(\ref{eq:SA}), so that we can apply the SA algorithm, as shown in Algorithm \ref{alg:JAE}, to find its root and thus solve the optimization problem Eq.(\ref{eq:JSA_unsup_obj}). 

\begin{prop}
	If Eq.(\ref{eq:JSA_unsup_gradient}) is solvable, then we can apply the SA algorithm to find its root.
\end{prop}

\begin{proof}
	This can be readily shown by first rewriting Eq.(\ref{eq:JSA_unsup_gradient}) as:
\begin{displaymath}
\left\{
\begin{split}
& \frac{1}{n} \sum_{k=1}^{n} E_{p_\theta(h_k|x_k)}\left[\nabla_\theta logp_\theta(x_k,h_k)\right] = 0 \\
& \frac{1}{n} \sum_{k=1}^{n} E_{p_\theta(h_k|x_k)}\left[ \nabla_\phi logq_\phi(h_k|x_k)\right] = 0 \\
\end{split}
\right.
\end{displaymath}	
	and then recasting it in the form of $f(\lambda) = 0$, with $\lambda \triangleq (\theta, \phi)^T$, $z \triangleq (k,h_1,\cdots,h_n)^T$, $p(z; \lambda) \triangleq \frac{1}{n} \prod_{j=1}^{n} p_\theta(h_j|x_j)$, and 
	\begin{displaymath}
	F(z; \lambda) \triangleq \left( \begin{array}{c}
	\nabla_\theta logp_\theta(x_k,h_k) \\
	\nabla_\phi logq_\phi(h_k|x_k)
	\end{array} \right),
	\end{displaymath}
	where $k$ denotes a uniform index variable over ${1,\cdots\,n}$.
\end{proof}

To apply the SA algorithm, we need to construct a Markov transition kernel $K_{\lambda}(z^{(t-1)},\cdot)$ that admits $p(z; \lambda)$ as the invariant distribution. There are many options. Particularly, we can use the Metropolis independence sampler (MIS), with $p(z; \lambda)$ as the target distribution. The proposal $q(z; \lambda)$ is defined by first drawing $k$ uniformly over ${1,\cdots\,n}$, and then only drawing $h_k \sim q_\phi(h_k|x_k)$ without changing other $h_j$ for $j \not= k$.
Given current sample $z^{(t-1)}$, MIS works as follow\footnote{We update one $h_k$ at a time so that in $\frac{w(z)}{w(z^{t-1})}$, we can cancel the intractable $p_{\theta}(x_k)$ which is appeared in the importance ratio $w(z)$. In practice, we run SA with multiple moves.
	
Note that although we can also recast Eq.(\ref{eq:JSA_unsup_gradient})
in the form of $f(\lambda) = 0$, with $\lambda \triangleq (\theta, \phi)^T$, $z \triangleq (x,h)^T$, $p(z; \lambda) \triangleq \tilde{p}(x) p_\theta(h|x) $, and 
\begin{displaymath}
F(z; \lambda) \triangleq \left( \begin{array}{c}
\nabla_\theta logp_\theta(x,h) \\
\nabla_\phi logq_\phi(h|x)
\end{array} \right),
\end{displaymath}
we can not easily construct a Markov transition kernel $K_{\lambda}(z^{(t-1)},\cdot)$ that admits $p(z; \lambda)$ as the invariant distribution. Drawing $x \sim \tilde{p}(x)$ and then using MIS to draw $p_\theta(h|x)$ with $q_\phi(h|x)$ as the proposal does not work, because the importance ratio $w(z) = \frac{p_\theta(h|x)}{q_\phi(h|x)}$ contains the intractable $p_{\theta}(x)$, and cannot be canceled out in $\frac{w(z)}{w(z^{(t-1)})}$ as $x \ne x^{(t-1)}$ in general.
}:
\begin{enumerate}
	
	\item Propose $z \sim q(z; \lambda)$, i.e. propose $k$ and then $h_k$,
	
	\item Accept $z^{t}=z$ with probability
	\begin{displaymath}
	min\left\lbrace 1, \frac{w(z)}{w(z^{(t-1)})} \right\rbrace, w(z) = \frac{p(z; \lambda)}{q(z; \lambda)} = \frac{p_\theta(h_k|x_k)}{q_\phi(h_k|x_k)}
	\end{displaymath}
\end{enumerate}

Since the parameters of the target and auxiliary models are jointly optimized based on the SA framework, the above method is referred to as \textbf{JSA learning}. 
It can be seen from the above derivation that JSA learning is general, which places no constrains on the handling of discrete variables for $x$ and $h$.
In the following, we provide more comments and comparisons with existing learning techniques.

First, note that as in JSA iterations, minimizing $KL(\tilde{p}(x) p_\theta(h|x)||\tilde{p}(x) q_\phi(h|x))$ w.r.t. $\phi$ encourages the inference model to chase the posteriori, which subsequently improves the sampling efficiency of using the inference model as the proposal for sampling the posteriori.
Also the inclusive KL divergence ensures that $q_\phi(h|x)>0$ wherever $p_\theta(h|x)>0$, which makes $q_\phi(h|x)$ a valid proposal for sampling $p_\theta(h|x)$.

Second, note that adversarial learning of GANs involves finding a Nash equilibrium to a two-player non-cooperative game. Gradient descent may fail to converge, as analyzed in \cite{imporveGAN}. 
In contrast, Eq.(\ref{eq:JSA_unsup_obj}) in JSA learning is not finding a Nash equilibrium, and thus is more stable.

Third, variational learning is to optimize the variational lower bound (V-LB) as shown in Eq. \ref{eq:VAE_obj}.
While the gradient w.r.t. $\theta$ is well-behaved, the trouble is that the gradient w.r.t. $\phi$ is known to have high variance. To address this problem, there are a lot of efforts, as discussed in Sec. \ref{sec:variational-learning}.

Fourth, note that JSA learning mainly seeks ML estimates of $\theta$, with an additional optimization over $\phi$. So JSA estimator of $\theta$ enjoys the same theoretical properties as ML estimator, even if $q_\phi$ has finite capacity. Furthermore, if both $p_\theta$ and $q_\phi$ have infinite capacity, we will obtain not only the perfect generative model but also the perfect inference model. The following proposition shows the theoretical consistency of JSA learning in the nonparametric limit.
\begin{prop}
	Suppose that $n \to \infty$, and $p_\theta(x,h)$ and $q_\phi(h|x)$ have infinite capacity, then we have
	(i) both KL divergences in Eq.(\ref{eq:JSA_unsup_obj}) can be minimized to attain zeros. 
	(ii) If both KL divergences in Eq.(\ref{eq:JSA_unsup_obj}) attain zeros at $(\theta^*, \phi^*)$, then we have $p_{\theta^*}(x) = p_0(x)$, $q_{\phi^*}(h|x) = p_{\theta^*}(h|x)$, $x \in \mathcal{X}$.
\end{prop}

\begin{proof}
	By the property of the KL divergence, and $\tilde{p}(x) \to p_0(x)$ as $n \to \infty$.
\end{proof}

Fifth, note that $p_\theta(x|h)$ is often termed the decoder, and $q_\phi(h|x)$ the encoder. They could be defined either with multiple stochastic hidden layers (Sec. \ref{sec:model-def-SBN}) such as SBNs \cite{Saul1996}, or with multiple deterministic hidden layers (Sec. \ref{sec:model-def-VAE}) such as in VAEs \cite{kingma2014auto-encoding}. JSA could be applied in both cases, resulting in Jsa AutoEncoders, or JAEs for short. 
A note is that for the JAEs defined in the second manner like VAEs, the storage for saving the latent codes per training observation is much reduced, as compared to the JAEs defined in the first manner like SBNs \cite{xu2016joint}.

Finally, JAEs provide a simple, consistent and principled way to handle both discrete and continuous variables in latent and observation space. 
Moreover, while structure mismatch between the encoder and decoder causes an irreducible biased gap from the data log-likelihood for VAEs, JAE learning is still consistent. 
The MIS accept/reject mechanism in JAE learning will compensate for the mismatch.
A series of experiments is conducted in \cite{JSA2}, which shows the superiority of JAEs such as being robust to structure mismatch between encoder and decoder, consistent handling of both discrete and continuous variables.

\begin{algorithm}[tb]
	\caption{\textbf{JSA learning}, represented as the SA algorithm with multiple moves}
	\label{alg:JAE}
	\begin{algorithmic}
		\REPEAT
		\STATE \underline{Monte Carlo sampling:}
		Draw a minibatch $\mathcal{M} \sim \tilde{p}(x) p_\theta(h|x)$, which could be achieved by drawing $(k,h_1,\cdots,h_n)^T \sim p(z; \lambda) \triangleq \frac{1}{n} \prod_{j=1}^{n} p_\theta(h_j|x_j)$, but update one $h_k$ at a time by MIS;
		\STATE \underline{SA updating:}
		Update $\theta$ by ascending: 
		$\frac{1}{|\mathcal{M}|} \sum_{(x,h) \sim \mathcal{M}}
		\nabla_\theta logp_\theta(x,h)$;\\
		Update $\phi$ by ascending: 
		$\frac{1}{|\mathcal{M}|} \sum_{(x,h) \sim \mathcal{M}}
		\nabla_\phi logq_\phi(h|x)$;
		\UNTIL{convergence}
	\end{algorithmic}
\end{algorithm}

\subsubsection{Adversarial learning}

Adversarial learning\footnote{However, conflicting terminology is in use: ``adversarial training'' is also used to refer to the design of neural network training to achieve insensitivity/invulnerability to the perturbation of images \cite{szegedy2013intriguing,goodfellow6572explaining,madry2017towards}.
\emph{Adversarial attack} refers to that we design slightly modified images (\emph{adversarial examples}) to fool the neural network classifier.
\emph{Adversarial defenses} are techniques that make neural networks resistant
to adversarial examples.
Adversarial training is a defense algorithm.}, exemplified by GANs \cite{goodfellow2014generative}, is mainly applied to learning implicit models, which is defined by a transformation $G_\theta(\epsilon)$ of a random noise $\epsilon$, as introduced in \ref{sec:model-def-GAN}.
The key technique is that of introducing an auxiliary model which acts like a discriminator and is optimized simultaneously.

\textbf{GAN Training.} Given a set of IID samples $\mathcal{D} = \left\lbrace x_1, \cdots, x_n \right\rbrace $ from an unknown true distribution
$p_0(x)$, in order to find $p_\theta(x)$ that best describes the true distribution, the GAN training is originally formulated as playing a two-player minimax game:
\begin{equation} \label{eq:GAN_obj}
\min_\theta \max_{\psi} \mathcal{F}_{GAN}(\theta,\psi) \triangleq E_{x \sim p_0(x)} \left[log D_\psi(x)\right]
+ E_{\epsilon \sim p(\epsilon)} \left[log \left( 1-D_\psi(G_{\theta}(\epsilon)) \right)\right],
\end{equation}
$D_\psi(x)$ represents the discriminator, which calculates the probability that $x$ comes from the data $p_0(x)$ rather than $p_\theta(x)$.
In adversarial training, we train $D_\psi(x)$ to maximize the probability of assigning the correct labels to both training examples and generated samples, and simultaneously train $G_\theta(\epsilon)$ to minimize the probability of the correct labeling of generated samples, that is, to \emph{fool} the discriminator.
For any fixed generator $G_\theta(\epsilon)$, the optimal discriminator is $D_{\psi^*}(x) = \frac{p_0(x)}{p_0(x)+p_\theta(x)}$. 
When the discriminator is optimal, the objective in Eq. \ref{eq:GAN_obj} becomes
\begin{equation} \label{eq:GAN-JS-obj}
\begin{split}
\mathcal{F}_{GAN}(\theta,\psi^*) &= E_{x \sim p_0(x)} \left[log \frac{p_0(x)}{p_0(x)+p_\theta(x)}\right]
+ E_{\epsilon \sim p(\epsilon)} \left[log \left( \frac{p_\theta(x)}{p_0(x)+p_\theta(x)} \right)\right]\\
&=-log(4) + 2 \cdot JS[p_0(x)||p_\theta(x)]
\end{split}
\end{equation}

Given enough capacity of both generator and discriminator and training time, we would like the training algorithm to find a saddle-point of the objective in Eq. \ref{eq:GAN_obj}, i.e. obtaining $p_\theta(x)=p_0(x)$.

It is noted in \cite{goodfellow2014generative} that in practice it is more advantageous to maximize $E_{\epsilon \sim p(\epsilon)} \left[log D_\psi(G_{\theta}(\epsilon)) \right]$ w.r.t. $\theta$, which is called the $log D$ trick, the non-saturating loss \cite{kurach2018gan}, or GAN-2 \cite{nowozin2016f-gan}.
This is motivated by the observation that in the early stages of training when $G_\theta(\epsilon)$ is not
sufficiently well fitted, $D_\psi(x)$ can saturate fast leading to weak gradients in $E_{\epsilon \sim p(\epsilon)} \left[log \left( 1-D_\psi(G_{\theta}(\epsilon)) \right)\right]$.
The $E_{\epsilon \sim p(\epsilon)} \left[log D_\psi(G_{\theta}(\epsilon)) \right]$ term, however, can provide stronger gradients and leads to the same fixed point.

It is shown in \cite{nowozin2016f-gan} that in the variational divergence minimization (VDM) framework, we can recover the GAN training objective and generalize it to arbitrary f-divergences.
The GAN training is a special case of this more general VDM framework, minimizing a divergence between the data distribution and generator distribution (related to the JS divergence but not exactly as in Eq. \ref{eq:GAN-JS-obj}).
As we know, the optimization criterion used has profound effect on the behavior of the optimized model \cite{theis2016a,Minka2005DivergenceMA}. This VDM framework is important for us to be able to choose the right criterion for a given application, since different applications require different trade-offs.

\textbf{$f$-divergence.} Statistical divergences such as the well-known Kullback-Leibler divergence measure the difference
between two given probability distributions.
The $f$-divergence between $p_0(x)$ and $p_\theta(x)$ on the domain $\mathcal{X}$, is defined as follows:
\begin{equation} 
\label{eq:f-divergence}
D_f\left[p_0 || p_\theta\right] \triangleq \int_{\mathcal{X}} p_\theta(x) f\left(\frac{p_0(x)}{p_\theta(x)} \right) dx.
\end{equation}
The generator function\footnote{We call it $f$-function in this paper in order to distinguish from the generator distribution.} $f:\mathcal{R}_{+} \to \mathcal{R}$ is a convex, lower-semicontinuous function satisfying $f(1)=0$, which ensures that $D_f\left[p_0 || p_0 \right]=0$ for any distribution $p_0$.
Different choices of $f$ recover a number of popular divergences as special cases in Eq. \ref{eq:f-divergence}, which can be found in \cite{nowozin2016f-gan}.

\textbf{Variational Estimation of $f$-divergence.}
Every convex, lower-semicontinuous function $f$ has a convex (Fenchel) conjugate function $f^\dagger$.
The pair $(f, f^\dagger)$ is dual to another in the sense that $f^{\dagger \dagger} = f$. We have
\begin{displaymath}
\left\{
\begin{split}
& f^\dagger(t) = \sup_{u \in dom_f} \left\lbrace ut-f(u)\right\rbrace  \\
& f(u) = \sup_{t \in dom_{f^\dagger}} \left\lbrace tu-f^\dagger(t)\right\rbrace \\
\end{split}
\right.
\end{displaymath}

Leveraging the above variational representation of $f$, we obtain a lower bound on the $f$-divergence:
\begin{equation} \label{eq:f-gan-bound}
\begin{split}
D_f\left[p_0 || p_\theta\right]&= \int_{\mathcal{X}} p_\theta(x) \sup_{t \in dom_{f^\dagger}} \left\lbrace t\frac{p_0(x)}{p_\theta(x)}-f^\dagger(t)\right\rbrace dx\\
&\geq \sup_{T \in \mathcal{T}} \left(\int_{\mathcal{X}} p_0(x) T(x) dx - \int_{\mathcal{X}} p_\theta(x) f^\dagger(T(x)) dx \right)\\
&=\sup_{T \in \mathcal{T}} \left(E_{p_0(x)} \left[T(x)\right] - E_{p_\theta(x)} \left[f^\dagger(T(x))\right]\right)
\end{split}
\end{equation}
where $T$ is an arbitrary class of functions $T: \mathcal{X} \to dom_{f^\dagger}$. The above derivation yields a lower bound
for two reasons: first, because of Jensen’s inequality when swapping the integration and supremum
operations. Second, the class of functions $T$ may contain only a subset of all possible functions.
By taking the variation of the lower bound in Eq. \ref{eq:f-gan-bound} w.r.t. $T$, we find that under mild conditions on $f$, the bound is tight for
\begin{equation} \label{eq:f-gan-optimal-T}
T^*(x) = f'\left(\frac{p_0(x)}{p_\theta(x)} \right),
\end{equation}
where $f'$ denotes the first-order derivative of $f$.

\textbf{Learning by Variational Divergence Minimization (VDM).}
We now can use the variational lower bound Eq. \ref{eq:f-gan-bound} on the $f$-divergence $D_f\left[p_0 || p_\theta\right]$ to estimate the implicit generative model $p_\theta(x)$ given the true distribution $p_0(x)$.
For this purpose, $T$ is the variational function, taking as input a sample and returning a scalar. We parametrize $T$ using a vector $\psi$ and write $T_\psi(x)$.
In \textbf{VDM learning}, we learn $p_\theta(x)$ by finding a saddle-point of the following objective function, where we minimize w.r.t. $\theta$ and maximize w.r.t. $\psi$,
\begin{equation} \label{eq:f-gan-obj1}
\min_\theta \max_{\psi} \mathcal{F}_f(\theta,\psi) \triangleq
E_{p_0(x)} \left[T_\psi(x)\right] - E_{p_\theta(x)} \left[f^\dagger(T_\psi(x))\right]
\end{equation}

Setting to zeros the gradients of $\mathcal{F}_f(\theta,\psi)$ w.r.t. $(\theta,\psi)$ and finding the root for the resulting system of simultaneous equations are necessary to finding the saddle points of $\mathcal{F}(\theta,\psi)$.
So we can apply the SA algorithm, as shown in Algorithm \ref{alg:VDM}, to find its root. This is also the algorithm that the Single-Step Gradient Method proposed in \cite{nowozin2016f-gan} takes.

\begin{algorithm}[tb]
	\caption{\textbf{VDM learning}, represented as the SA algorithm with multiple moves}
	\label{alg:VDM}
	\begin{algorithmic}
		\REPEAT
		\STATE \underline{Monte Carlo sampling:}
		Draw a empirical minibatch $\mathcal{M} \sim {p_0}(x)$, a generated minibatch $\mathcal{E} \sim p(\epsilon)$ so that $G_\theta(\epsilon) \sim p_\theta(x)$;
		\STATE \underline{SA updating:}\\
		Update $\psi$ by ascending:
		$\frac{1}{|\mathcal{M}|} \sum_{x \sim \mathcal{M}} \nabla_\psi T_\psi(x)
		-\frac{1}{|\mathcal{E}|} \sum_{\epsilon \sim \mathcal{E}} \nabla_\psi f^\dagger\left(T_\psi(G_\theta(\epsilon))\right)$;\\
		Update $\theta$ by descending: 
		$- \frac{1}{|\mathcal{E}|} \sum_{\epsilon \sim \mathcal{E}}
		\nabla_\theta f^\dagger\left(T_\psi(G_\theta(\epsilon))\right)$;
		\UNTIL{convergence}
	\end{algorithmic}
\end{algorithm}

\textbf{Representation for the Variational Function $T_\psi$.}
To apply the above variational objective, we need to respect the domain $dom_{f^\dagger}$ of the conjugate functions $f^\dagger$. To this end, we \emph{assume} that variational function $T_\psi$ is
represented in the form $T_\psi(x)=g_f(V_\psi(x))$ and rewrite the saddle objective Eq. \ref{eq:f-gan-obj1} as follows:
\begin{equation} \label{eq:f-gan-obj2}
\mathcal{F}_f(\theta,\psi) \triangleq
E_{p_0(x)} \left[g_f(V_\psi(x))\right] - E_{p_\theta(x)} \left[f^\dagger(g_f(V_\psi(x)))\right]
\end{equation}
where $V_\psi: \mathcal{X} \to \mathcal{R}$ without any range constraints on the output, and $g_f: \mathcal{R} \to dom_{f^\dagger}$ is an output activation function specific to the $f$-divergence used. 
Although not mandatory, we usually choose $g_f$ to be monotone increasing functions so that a large output $V_\psi(x)$ corresponds to the belief of the variational function that the sample $x$ comes from the data distribution $p_0$ as in the GAN case.
It is also instructive to look at the second term
$f^\dagger(g_f(V_\psi(x)))$ in the above saddle objective. 
This term is typically (except for the Pearson $\chi^2$ divergence) a decreasing function of the output $V_\psi(x)$.

Further, note that Eq. \ref{eq:f-gan-optimal-T} shows the link between the variational function $T_\psi(x)$ and the density ratio $\frac{p_0(x)}{p_\theta(x)}$. The critical value $f'(1)$ can be interpreted as a classification threshold applied to $T(x)$ to distinguish between true and generated samples.
We can interpret the output of the variational function $T_\psi(x)$ as classifying the input $x$ as a true sample if the variational function $T_\psi(x)$ is larger than $f'(1)$, and classifying it as a sample from the generator otherwise.

In the following, we present three specific cases of $f$-divergence for VDM learning, which correspond to the GAN objective, KL-divergence and reverse-KL-divergence, and then introduce a principled method to choose $g_f(v)$.
\begin{itemize}
\item \textbf{VDM learning using the GAN objective recovers the GAN training.}

Choosing $f$-function as $f(u) = u log u - (u+1)log(u+1)$, then we have the conjugate $f^\dagger(t) = -log(1-e^{t})$ with $dom_{f^\dagger} = \mathcal{R}_{-}$.
Substituting into the $f$-divergence definition (Eq. \ref{eq:f-divergence}), we have the following specific $f$-divergence
\begin{displaymath}
D_f\left[p_0 || p_\theta\right] = E_{p_0(x)} \left[log \frac{p_0(x)}{p_0(x)+p_\theta(x)}\right]
+ E_{p_\theta(x)} \left[log \left( \frac{p_\theta(x)}{p_0(x)+p_\theta(x)} \right)\right].
\end{displaymath}
Further, choosing $g_f(v)=-log(1+e^{-v})=log \sigma(v)$ respects the right $dom_{f^\dagger} = \mathcal{R}_{-}$, and $T_\psi(x)=log \sigma(V_\psi(x))$.

Substituting these into the VDM objective (Eq. \ref{eq:f-gan-obj2}), we have
\begin{displaymath}
\begin{split}
\mathcal{F}_{GAN}(\theta,\psi) &=
E_{p_0(x)} \left[ -log\left(1+e^{-V_\psi(x)} \right) \right] 
- E_{p_\theta(x)} \left[ -log \left(1-e^{-log \left( 1+e^{-V_\psi(x)}  \right)} \right) \right]\\
&= E_{p_0(x)} \left[log D_\psi(x)\right]
+ E_{p_\theta(x)} \left[log \left( 1-D_\psi(x) \right)\right],
\end{split}
\end{displaymath}
where we define $D_\psi(x) \triangleq \frac{1}{1+e^{-V_\psi(x)}}$, and thus have $V_\psi(x) = log \frac{D_\psi(x)}{1-D_\psi(x)}$. The above exactly corresponds to the GAN objective (Eq. \ref{eq:GAN_obj}).

To interpret $V_\psi(x)$, according to Eq. \ref{eq:f-gan-optimal-T}, we have the optimal variational function
\begin{displaymath}
T^*(x) = f'\left(\frac{p_0(x)}{p_\theta(x)} \right) 
= log \left( \frac{p_0(x)}{p_0(x)+p_\theta(x)} \right),
\end{displaymath}
which gives the optimal $V_{\psi^*}(x) = log \frac{p_0(x)}{p_\theta(x)}$. This helps us to understand the discriminative meaning of $V_{\psi}(x)$, apart from being the variational function.

\item \textbf{VDM learning using the KL-divergence.}

Choosing $f$-function as $f(u) = u log u$, then we have the conjugate $f^\dagger(t) = e^{t-1}$ with $dom_{f^\dagger} = \mathcal{R}$.
Substituting into the $f$-divergence definition (Eq. \ref{eq:f-divergence}), we have the following specific $f$-divergence, which is exactly the KL-divergence
\begin{displaymath}
D_f\left[p_0 || p_\theta\right] = E_{p_0(x)} \left[log \frac{p_0(x)}{p_\theta(x)}\right].
\end{displaymath}
Further, choosing $g_f(v)=v$ respects the right $dom_{f^\dagger} = \mathcal{R}$, and $T_\psi(x)=V_\psi(x)$.

Substituting these into the VDM objective (Eq. \ref{eq:f-gan-obj2}), we have
\begin{displaymath}
\mathcal{F}_{KL}(\theta,\psi) =
E_{p_0(x)} \left[ V_\psi(x) \right] 
- E_{p_\theta(x)} \left[ e^{V_\psi(x)-1} \right].
\end{displaymath}

To interpret $V_\psi(x)$, according to Eq. \ref{eq:f-gan-optimal-T}, we have the optimal variational function
\begin{displaymath}
T^*(x) = f'\left(\frac{p_0(x)}{p_\theta(x)} \right) 
= 1+log \frac{p_0(x)}{p_\theta(x)},
\end{displaymath}
which gives the optimal $V_{\psi^*}(x) = 1+log \frac{p_0(x)}{p_\theta(x)}$.
This helps us to understand the discriminative meaning of $V_{\psi}(x)$, apart from being the variational function.

Additionally, it is instructive to show that choosing $g_f(v)=v+1$ is also valid. Then we have $T_\psi(x)=V_\psi(x)+1$, and 
\begin{displaymath}
\mathcal{F}_{KL}(\theta,\psi) =
E_{p_0(x)} \left[ V_\psi(x)+1 \right] 
- E_{p_\theta(x)} \left[ e^{V_\psi(x)} \right].
\end{displaymath}
In this case, the optimal $V_{\psi^*}(x) = log \frac{p_0(x)}{p_\theta(x)}$ directly equals to the log density ratio.

\item \textbf{VDM learning using the reverse-KL-divergence.}

Choosing $f$-function as $f(u) = - log u$, then we have the conjugate $f^\dagger(t) = -1-log(-t)$ with $dom_{f^\dagger} = \mathcal{R}_{-}$.
Substituting into the $f$-divergence definition (Eq. \ref{eq:f-divergence}), we have the following specific $f$-divergence, which is exactly the reverse-KL-divergence
\begin{displaymath}
D_f\left[p_\theta || p_0\right] = E_{p_\theta(x)} \left[log \frac{p_\theta(x)}{p_0(x)}\right].
\end{displaymath}
Further, choosing $g_f(v)=-e^{-v}$ respects the right $dom_{f^\dagger} = \mathcal{R}_{-}$, and $T_\psi(x)=-e^{-V_\psi(x)}$.

Substituting these into the VDM objective (Eq. \ref{eq:f-gan-obj2}), we have
\begin{displaymath}
\mathcal{F}_{rKL}(\theta,\psi) =
E_{p_0(x)} \left[ -e^{-V_\psi(x)} \right] 
- E_{p_\theta(x)} \left[ -1 + V_\psi(x) \right].
\end{displaymath}

To interpret $V_\psi(x)$, according to Eq. \ref{eq:f-gan-optimal-T}, we have the optimal variational function
\begin{displaymath}
T^*(x) = f'\left(\frac{p_0(x)}{p_\theta(x)} \right) 
= -\frac{p_\theta(x)}{p_0(x)},
\end{displaymath}
which gives the optimal $V_{\psi^*}(x) = log\frac{p_0(x)}{p_\theta(x)}$.
This helps us to understand the discriminative meaning of $V_{\psi}(x)$, apart from being the variational function.

\item \textbf{A principled method to choose $g_f(v)$.}

The $g_f: \mathcal{R} \to dom_{f^\dagger}$ is not uniquely determined. 
Here we provide a principled method to choose $g_f(v)$, by which you can easily obtain the choice of $g_f$ in Table 2 in \cite{nowozin2016f-gan}, which seems to be arbitrary. This method is concurrently appeared in \cite{CAL-GAN}.

This method is based on a property for convex and differentiable functions (Section 3.3.2 in \cite{convex-opt-book}). According to this property, we have
\begin{displaymath}\label{eq:conjugate-formula}
f^\dagger (f^\prime(u) ) = u f^\prime (u) - f(u),
\end{displaymath}
for the $f:\mathcal{R}_{+} \to \mathcal{R}$ in defining $f$-divergence, which is convex and differentiable.
Rewriting the $f$-GAN objective Eq. \ref{eq:f-gan-obj2} as follow:
\begin{equation} \label{eq:f-gan-obj3}
\mathcal{F}_f(\theta,\psi) =
E_{p_0(x)} \left[g_f(v)\right] - E_{p_\theta(x)} \left[f^\dagger(g_f(v))\right]~\text{with}~v=V_\psi(x),
\end{equation}
we see that a natural choice for $g_f(v)$ is to let
\begin{displaymath}
g_f(v) = \left\lbrace  f'(u) \right\rbrace |u=e^v.
\end{displaymath}
Substituting this choice into Eq. \ref{eq:f-gan-obj3}, we have
\begin{equation} \label{eq:f-gan-obj4}
\mathcal{F}_f(\theta,\psi) =
E_{p_0(x)} \left[ f'(e^v) \right] - E_{p_\theta(x)} \left[ e^v f'(e^v) - f(e^v) \right]~\text{with}~v=V_\psi(x).
\end{equation}
Further, according to Eq. \ref{eq:f-gan-optimal-T}, we have the optimal variational function
\begin{displaymath}
g_f(V_{\psi^*}(x)) = f'\left(\frac{p_0(x)}{p_\theta(x)} \right),
\end{displaymath}
which gives $V_{\psi^*}(x) = log\frac{p_0(x)}{p_\theta(x)}$.

\end{itemize}

The above gives the basics of adversarial learning, from the classic GAN training to the general VDM learning.
In the following, we present existing problems, recent advances and open questions related to adversarial learning.

\begin{enumerate}
	\item
	\textbf{Improve/Stabilize training, by exploring different training criteria.}
	
	Training GANs is well known for being delicate and unstable, for reasons theoretically investigated in \cite{Arjovsky2017TowardsPM}.
	It is noted that the various ways to measure how close the real distribution $p_0$ and the model distribution $p_\theta$ are, or equivalently, the various ways to define a divergence $D\left[p_0 || p_\theta\right]$ have different impact on the convergence of sequences of probability distributions.
	It is suggested to minimize the Wasserstein distance between the data distribution and the generator distribution, used by \textbf{Wasserstein GANs (W-GANs)} \cite{arjovsky2017wasserstein}.
	The W-GAN minimizes
\begin{equation} \label{eq:wgan-obj}
	\min_\theta \max_{||T_\psi||_L \le 1} \mathcal{F}_{WGAN}(\theta,\psi) \triangleq
	E_{x \sim p_0(x)} \left[T_\psi(x)\right] - E_{\epsilon \sim p(\epsilon)} \left[T_\psi(G_\theta(\epsilon))\right]
\end{equation}
	where $T_\psi(x)$ is called the \emph{critic}\footnote{Note that $T_\psi(x)$ is also called the ``discriminator'' in some papers, although it is actually a real-valued function and not a classifier at all.} or \emph{value function} in WGANs.
	The maximum is taken over the set of 1-Lipschitz functions.

	In addition to objective functions, adverarial training involves a number of considerations, including regularization schemes, network architectures, etc \cite{kurach2018gan}. In the following we introduce some recent strong developments.

	\item
	\textbf{Improve/Stabilize training, by exploring different regularizations.}
	\begin{itemize}
		\item 
			\textbf{Spectral Normalization GANs (SN-GANs)} \cite{Miyato2018SpectralNF}. The motivation is to introduce some form of restriction to the choice of the discriminator, namely to regularize the discriminator. The core is a new weight normalization technique, which normalizes the spectral norms of all weight matrices in the discriminator network.
		This is equivalent to constrain the discriminator to be $K$-Lipschitz continuous ($K=1$)\footnote{The Lipschitz constant of a general differentiable function $f:\mathcal{R}^n \rightarrow \mathcal{R}^m$ is the maximum spectral norm (maximum singular value) of its gradient operator $\nabla f$ (which is essentially a matrix operator).
			The maximum singular value of a matrix $A$ is equal to the operator norm of $A$.}.
		It is shown in \cite{Miyato2018SpectralNF} that when \emph{regularizing the discriminator}, SN outperforms other regularization/normalization techniques in terms of superior inception scores and FIDs, including weight clipping \cite{Arjovsky2017TowardsPM}, WGAN-GP/gradient penalty \cite{Gulrajani2017ImprovedTO}, batch-normalization (BN) \cite{Ioffe2015BatchNA}, layer normalization (LN) \cite{ba2016layer}, weight normalization (WN) \cite{salimans2016weight} and orthonormal regularization \cite{brock2016neural}.
	\end{itemize}
	
	\item
	\textbf{Improve/Stabilize training, by exploring different network architectures.}
	\begin{itemize}
	\item
	A new discriminator network architecture, called the \textbf{projection-based discriminator}, is proposed in \cite{miyato2018cgans} for conditional GANs (cGANs).
	This is in contrast with most conditional GANs, in which the discriminator uses the conditional information by concatenating the (embedded) conditional vector to the feature vectors.
	Combining the projection-based discriminator and spectral normalization of the discriminator greatly improves class-conditional image generation.
	For the first time, GANs (acting as cGANs) learn to generate high quality images from the full 1000-class $128 \times 128$ ImageNet dataset data with only one generator-discriminator pair \cite{miyato2018cgans}\footnote{For auxiliary classifier (AC) GANs \cite{odena2016conditional}, the authors prepared a pair of discriminator and generator for each set classes of size 10.}.

	\item
	\textbf{Self-Attention GANs (SA-GANs)} \cite{zhang2018self}.
	In image generation, convolutional GANs could easily generate images with a simpler geometry like Ocean, Sky etc., but failed on images with some specific geometry like dogs, horses and many more. Convolutional GANs were able to generate the texture of furs of dog but were unable to generate distinct legs. 
	This problem is arising because the convolution is a local operation whose receptive field depends on the spatial size of the kernel. 
	In a convolution operation, it is not possible for an output on the top-left position to have any relation to the output at bottom-right. 
	The motivation of SA-GANs is to create a balance between long-range dependencies(= large receptive fields) and efficiency via self-attention.
	SA-GANS used this self-attention layer in \emph{both the generator and discriminator}.
	Moreover, 
	SA-GANs applied spectral normalization (SN) to the weights in \emph{both generator and discriminator}, unlike the previous paper \cite{Miyato2018SpectralNF} which only normalizes the discriminator weights.	
	SA-GANs used a two-timescale update rule (TTUR) \cite{heusel2017gans} which is simply using different learning rate for both discriminator and generator.
	Using self-attention, spectral normalization, TTUR, and projection-based discriminator in combination, SA-GANs (acting as cGANs) further beat \cite{miyato2018cgans} in class-conditional image generation on ImageNet.
		
	\end{itemize}

	\item
	\textbf{Lack of inference mechanism.}
	
Remarkably, GANs lack the ability to infer the latent variable given the observation, and this limitation has been addressed by some recent studies ALI \cite{dumoulin2016adversarially}, BiGAN \cite{Donahue2016AdversarialFL}, WAE \cite{Tolstikhin2017WassersteinA}.

	\item
	\textbf{Handling of discrete latent variables.}

Learning GANs with discrete hidden variables remains unexplored. 
	
	\item
	\textbf{Handling of discrete observations.}

For GANs to work with discrete observations, it is difficult to propagating gradients back from the discriminator through the generated samples to the generator.
The application of GANs to discrete data is restricted yet with some efforts. 
\cite{gumbelgan} resorts to the Gumbel-softmax distribution; \cite{yu2016seqgan} (SeqGAN) and \cite{maligan} (MaliGAN) models the generation of the discrete sequence as a stochastic policy in reinforcement learning and perform gradient policy update.
	
	\item
	\textbf{Conditional generation.}
	\begin{itemize}
	\item		
	Beyond of learning (unconditional) implicit generative models, adversarial learning has also been successfully applied in the conditional generation setting, i.e. learning conditional distributions $p_\theta(x|c)$, or intuitively say, mappings.
	GANs are implicit generative models that learn a mapping from random noise vector $\epsilon$ to output $x$: $G_\theta (\cdot) : \epsilon \to x$. In contrast, \textbf{conditional GANs (cGANs)} learn a mapping from random noise vector $\epsilon$ and additional input $c$ to $x$: $G_\theta (\cdot) : (c, \epsilon) \to x$ \cite{mirza2014conditional}. 
	$x$ could be any kind of auxiliary/side information, such as class labels or data from other modalities.
	The objective of a conditional GAN can be expressed as:
	\begin{displaymath}
	\mathcal{F}_{cGAN}(\theta,\psi) \triangleq E_{(c,x) \sim p_0(c,x)} \left[log D_\psi(c,x)\right]
	+ E_{c \sim p_0(c), \epsilon \sim p(\epsilon)} \left[log \left( 1-D_\psi(c,G_{\theta}(c,\epsilon)) \right)\right].
	\end{displaymath}
	It can be seen that the above objective arises naturally from solving the following optimization based on VDM
	\begin{displaymath}
	\min_\theta E_{c \sim p_0(c)} \left\lbrace  D_{GAN}\left[p_0(x|c) || p_\theta(x|c) \right] \right\rbrace 
	\end{displaymath}
	The techniques of improve/stabilize training as introduced in the above unconditional setting mostly applies in the conditional setting.
	Important new issues involve the network architectures in implementing the conditional generator $G_\theta(c,\epsilon)$ and the discriminator $D_\psi(c,x)$.
	The discriminator in cGANs uses the conditional information $c$ by concatenating the (embedded) conditional vector $c$ and the input observation $x$.
	A new discriminator network architecture, called the projection-based discriminator, is proposed in \cite{miyato2018cgans} for conditional GANs, as we introduce before.
	Besides class-conditional image generation, many other tasks has been addressed in learning mappings with cGANs, such as image-to-image translation \cite{isola2017image}.
	\item
	Instead of concatenating class label $c$ and input observation $x$ to the discriminator as used in cGANs, \textbf{Auxiliary Classifier (AC) GANs} \cite{odena2016conditional} augment the discriminator to classify the observation $x$ into one of $K$ possible classes, in addition to discriminate real/fake samples. 
	In implementation, the output layer of $D_\psi(c,x)$ is of size $K+1$, which consists of a Softmax output over $K$ of the units ($1,\cdots, K$) and a Sigmoid output to the remaining unit $K+1$ (representing the probability of the sample being true). The objective function has two parts: the log-likelihood of the correct source, $\mathcal{F}_S$, and the log-likelihood of the correct class, $\mathcal{F}_C$.
	\begin{displaymath}
	\begin{split}
	\mathcal{F}_S &= E_{(c,x) \sim p_0(c,x)} \left[log D_\psi(K+1,x)\right]
	+ E_{c \sim p_0(c), \epsilon \sim p(\epsilon)} \left[log \left( 1-D_\psi(K+1,G_{\theta}(c,\epsilon)) \right)\right]\\
	\mathcal{F}_C &= E_{(c,x) \sim p_0(c,x)} \left[log D_\psi(c,x)\right]
	+ E_{c \sim p_0(c), \epsilon \sim p(\epsilon)} \left[log D_\psi(c,G_{\theta}(c,\epsilon)) \right]
	\end{split}
	\end{displaymath}
	$D_\psi$ is trained to maximize $\mathcal{F}_S+\mathcal{F}_C$, while $G_\theta$ is trained to maximize $-\mathcal{F}_S+\mathcal{F}_C$.
	Remarkably, AC-GANs learn a representation for $\epsilon$ that is independent of class label.
	\end{itemize}

\end{enumerate}

\section{Learning with Deep Undirected Models}

Generally speaking, there are two broad classes of DGMs - directed and undirected.
There emerged a bundle of deep directed generative models, such as Helmholtz Machines \cite{dayan1995helmholtz}, Variational Autoencoders (VAEs) \cite{kingma2014auto-encoding}, Generative Adversarial Networks (GANs) \cite{goodfellow2014generative}, auto-regressive neural networks \cite{Larochelle2011TheNA} and so on.
In contrast, undirected generative models (also known as random fields \cite{koller2009probabilistic}, energy-based models \cite{energy-based}), e.g. Deep Boltzmann Machines (DBMs) \cite{dbm}, received less attentions with slow progress.
This is presumably because fitting undirected models is more challenging than fitting directed models \cite{koller2009probabilistic}.
In general, calculating the log-likelihood and its gradient is analytically intractable, because this involves evaluating the normalizing constant (also called the partition function in physics) and, respectively, the expectation with respect to the model distribution.

\subsection{Model Definition}

Roughly speaking, there are two types of deep undirected models (deep random fields) in the literature. Those with multiple stochastic hidden layers such as deep belief networks (DBNs) \cite{hinton2006a} and deep Boltzmann machines (DBMs) \cite{dbm} involve very difficult inference and learning, which severely limits their applications beyond of the form of pre-training.
Another type, which appears to be more successfully applied,  
is to directly define the potential function\footnote{Negating the potential function defines the energy function.} through a deep neural network. In this case, the layers of the network do not represent latent variables but rather are deterministic transformations of input observations. 

The second type of deep random fields has been proposed several times in different contexts. They are once called deep energy models (DEMs) in \cite{ng11,Kim2016DeepDG}, descriptive models in \cite{descriptor,coopnets}, generative ConvNet in \cite{wyn15}, neural trans-dimensional random field language models in \cite{asru,Bin2018,wang2018improved}, neural random fields in \cite{song2018learning}. 
There are some specific differences between definitions of these models. \cite{Kim2016DeepDG} includes linear and squared terms in $u_{\theta}(x)$, \cite{asru} defines over sequences, and \cite{coopnets} defines in the form of exponential tilting of a reference distribution (a Gaussian white noise distribution).
However, when presenting learning algorithms, these differences would not affect much.
In the following, these models are collectively referred to as random fields (RFs).

Generally, consider a random field for modeling observation $x$ with parameter $\theta$:
\begin{equation}\label{eq:unsup-RF}
p_{\theta}(x)=\frac{1}{Z(\theta)} \exp\left[  u_{\theta}(x) \right] 
\end{equation}
where $Z(\theta)=\int\exp(f(x;\theta))dx$ is the normalizing constant, $u_{\theta}(x)$ is the potential function which assigns a scalar value to each configuration of random variable $x$.
The general idea of defining a neural random field (NRF) is to implement $u_{\theta}(x)$ by a neural network, taking $x$ as input and outputting $u_{\theta}(x)$, so that we can take advantage of the representation power of neural networks.
It is usually intractable to maximize the data log-likelihood $log p_\theta(\tilde{x})$ for observed $\tilde{x}$, since the gradient involves expectation w.r.t. the model distribution, as shown below:
\begin{displaymath}
\nabla_\theta \log{p}_{\theta}(\tilde{x})=\nabla_\theta u_{\theta}(\tilde{x})-E_{p_\theta(x)}\left[\nabla_\theta u_{\theta}(x)\right]
\end{displaymath}

\subsection{Model Learning}

There is an extensive literature devoted to maximum likelihood (ML) learning of random fields (RFs).
An important class of RF learning methods is stochastic approximation (SA) methods \cite{SA51}, which approximates the model expectations by Monte Carlo sampling for calculating the gradient. 
Basically, SA training iterates Monte Carlo sampling and SA update of parameters.
The classic algorithm, initially proposed in \cite{younes1989parametric}, is often called \textbf{stochastic maximum likelihood (SML)}.
In the literature on training restricted Boltzmann machines (RBMs), SML is also known as persistent contrastive divergence (PCD) \cite{tieleman2008training} to emphasize that the Markov chain is not reset between parameter updates.

For both types of deep random fields, the classic learning algorithm is SML.
We are mainly concerned with studying the second type of deep random fields in this review,  which has been more successfully applied.
In \cite{ng11}, contrastive divergence (CD) \cite{cd} is used for training DEMs, which, in contrast to PCD, is biased in theory since it performs Monte Carlo sampling from the training observations.
A recent progress as studied in \cite{Kim2016DeepDG,asru,coopnets} is to pair the target random field with an auxiliary directed generative model (often called generator), which approximates the target random field but is easy to do sampling\footnote{Ancenstral sampling from a directed graphical model is straightforward.}.
Learning is performed by maximizing the target data log-likelihood and simultaneously minimizes some divergence between the target random field and the auxiliary generator.

The learning setting is the same as in learning with deep directed Models - Maximum likelihood learning.
Suppose that data $\mathcal{D} = \left\lbrace \tilde{x}_1, \cdots, \tilde{x}_n \right\rbrace $, consisting of $n$ IID observations, are drawn from the true but unknown data distribution $p_0(\cdot)$.
$\tilde{p}(\tilde{x}) \triangleq \frac{1}{n} \sum_{k=1}^{n} \delta(\tilde{x} - \tilde{x}_n)$ denotes the empirical data distribution.
There exist several representative classes of learning algorithms, as introduced below.
Also note that these algorithms admit natural extensions to conditional random field (CRF) undirected models.

\subsubsection{Stochastic Maximum Likelihood (SML) learning}

By setting the gradients to zeros, the maximum learning problem Eq. \ref{eq:MLE} can be solved by finding the root for the following equation:
\begin{equation}  \label{eq:SML}
E_{\tilde{p}(\tilde{x})}\left[ \nabla_\theta logp_\theta(\tilde{x})\right]
=E_{\tilde{p}(\tilde{x})}\left[\nabla_\theta u_\theta(\tilde{x})\right]-E_{p_\theta(x)}\left[\nabla_\theta u_\theta(x)\right]=0.
\end{equation}

It can be shown that Eq.(\ref{eq:SML}) exactly follows the form of Eq.(\ref{eq:SA}), so that we can apply the SA algorithm, as shown in Algorithm \ref{alg:VAE}, to find its root and thus solve the optimization problem Eq.(\ref{eq:SML}).
A crucial step in the classic SML algorithm is that we need to draw samples from the target random field $p_\theta(x)$.

\begin{algorithm}[tb]
	\caption{\textbf{SML learning}, represented as the SA algorithm with multiple moves}
	\label{alg:VAE}
	\begin{algorithmic}
		\REPEAT
		\STATE \underline{Monte Carlo sampling:}
		Draw an empirical minibatch $\mathcal{E} \sim \tilde{p}(\tilde{x})$ and a Monte Carlo minibatch $\mathcal{M} \sim p_\theta(x)$;
		\STATE \underline{SA updating:}
		Update $\theta$ by ascending: 
		$\frac{1}{|\mathcal{E}|} \sum_{\tilde{x} \sim \mathcal{E}}
		\nabla_\theta u_\theta(\tilde{x}) - \frac{1}{|\mathcal{M}|} \sum_{x \sim \mathcal{M}}
		\nabla_\theta u_\theta(x)$;
		\UNTIL{convergence}
	\end{algorithmic}
\end{algorithm}

\subsubsection{Learning NRFs with Exclusive Auxiliary Generators (Exclusive-NRFs)}

We can pair the target random field $p_{\theta}(x)$ with an auxiliary generator $q_\phi(x)$ and jointly train the two models.
This idea has been studied in \cite{Kim2016DeepDG} by minimizing the \emph{exclusive} KL divergence $KL(q_\phi(x)||p_\theta(x))$ w.r.t. $\phi$, and in \cite{asru} by minimizing the \emph{inclusive} KL divergence $KL(p_\theta(x)||q_\phi(x))$ w.r.t. $\phi$, in addition to maximizing w.r.t. $\theta$ the data log-likelihood.

The generator $q_\phi$, which is usually a directed generative model (Section \ref{sec:directed-model-def}), could be implemented as an implicit density like in Eq. \ref{eq:GAN} \cite{Kim2016DeepDG}, a prescribed density $q_\phi(x,h) \triangleq q(h)q_\phi(x|h)$ \cite{song2018learning}, or an autoregressive model like an LSTM-RNN \cite{asru}.

For learning a NRF $p_\theta(x)$ with an exclusive auxiliary generator $q_\phi(x)$, we jointly optimize
\begin{equation}
\label{eq:EAG}
\left\{
\begin{split}
& \min_{\theta} KL\left[  \tilde{p}(\tilde{x}) || p_\theta(\tilde{x}) \right] \\
& \min_{\phi} KL\left[  q_\phi(x) || p_\theta(x) \right] \\
\end{split}
\right.
\end{equation}
By setting the gradients to zeros, the above optimization problem can be solved by finding the root for the following system of simultaneous equations:
\begin{equation}
\label{eq:EAG_gradient}
\left\{
\begin{split}
& E_{\tilde{p}(\tilde{x})}\left[\nabla_\theta logp_\theta(\tilde{x})\right]
= E_{\tilde{p}(\tilde{x})}\left[\nabla_\theta u_\theta(\tilde{x})\right]-E_{p_\theta(x)}\left[\nabla_\theta u_\theta(x)\right]=0\\
&\nabla_\phi \left\lbrace  -E_{q_\phi(x)}\left[ logp_\theta(x)\right]
- H \left[ q_\phi(x) \right] \right\rbrace =0\\
\end{split}
\right.
\end{equation}
The exclusive divergence includes a reconstruction term and an entropy term.
Fortunately, we have $\nabla_\phi logp_\theta(x) = \nabla_\phi u_\theta(x)$, i.e. this gradient does not depend on the partition function of the target random field, and we can use the reparameterization trick to calculate the gradient of the reconstruction term w.r.t. $\phi$.
Unfortunately, calculating the gradient of the entropy term is generally intractable, and it is not easy to draw samples from the target random field $p_\theta(x)$.

In \cite{Kim2016DeepDG}, two approximations are used to handle the two annoying terms.
First, $E_{p_\theta(x)}\left[\nabla_\theta u_\theta(x)\right]$ is approximated by $  E_{q_\phi(x)}\left[\nabla_\theta u_\theta(x)\right]$.
Second, note that the batch normalization \cite{Ioffe2015BatchNA} used in implementing $q_\phi(x)$ maps each activation $a_i$ into an approximately normal distribution with trainable mean-related shift parameter $\mu_{a_i}$ and variance-related scale parameter $\sigma^2_{a_i}$, and the entropy of the normal distribution over each activation can be measured analytically
$H[\mathcal{N}(\mu_{a_i}, \sigma^2_{a_i})] = \frac{1}{2}log(2 e \pi \sigma^2_{a_i})$.
It is assumed that this (internal) entropy could effect the entropy of the generator (external) distribution $H \left[ q_\phi(x) \right]$, which therefore could be approximated as: 
\begin{displaymath}
H \left[ q_\phi(x) \right] \approx \sum_{a_i} H[\mathcal{N}(\mu_{a_i}, \sigma_{a_i})],
\end{displaymath}
where the sum is over all activations in the generator's network.


\subsubsection{Learning NRFs with Inclusive Auxiliary Generators (Inclusive-NRFs)}

As shown above, the exclusive divergence includes a reconstruction term and an entropy term, which resembles the expression of the exclusive KL divergence in the variational learning. The entropy term is analytically intractable, so an ad hoc approximation is used in \cite{Kim2016DeepDG} without strict justification.
In contrast, leaning with inclusive auxiliary generator will not have such annoying entropy term.

For learning a NRF $p_\theta(x)$ with an inclusive auxiliary generator, e.g. implemented as a prescribed density $q_\phi(x,h) \triangleq q(h)q_\phi(x|h)$ \cite{song2018learning}, we jointly optimize
\begin{equation}
\label{eq:jrf_unsup_obj}
\left\{
\begin{split}
& \min_{\theta} KL\left[  \tilde{p}(\tilde{x}) || p_\theta(\tilde{x}) \right] \\
& \min_{\phi} KL\left[  p_\theta(x) || q_\phi(x) \right] \\
\end{split}
\right.
\end{equation}
By setting the gradients to zeros, the above optimization problem can be solved by finding the root for the following system of simultaneous equations:
\begin{equation}
\label{eq:jrf_unsup_gradient}
\left\{
\begin{split}
& E_{\tilde{p}(\tilde{x})}\left[\nabla_\theta logp_\theta(\tilde{x})\right]
=E_{\tilde{p}(\tilde{x})}\left[\nabla_\theta u_\theta(\tilde{x})\right]-E_{p_\theta(x)}\left[\nabla_\theta u_\theta(x)\right]=0\\
&E_{p_\theta(x)}\left[ \nabla_\phi logq_\phi(x)\right]
=E_{p_\theta(x) q_\phi(h|x)}\left[ \nabla_\phi logq_\phi(x,h)\right]=0\\
\end{split}
\right.
\end{equation}
It can be shown that Eq.(\ref{eq:jrf_unsup_gradient}) exactly follows the form of Eq.(\ref{eq:SA}), so that we can apply the SA algorithm to find its root and thus solve the optimization problem Eq.(\ref{eq:jrf_unsup_obj}). 
The SA algorithm with multiple moves is shown in Algorithm \ref{alg:learning-NRF-IAG}.

\begin{algorithm}[tb]
	\caption{Learning NRFs with inclusive auxiliary generators}
	\label{alg:learning-NRF-IAG}
	\begin{algorithmic}
		\REPEAT
		\STATE \underline{Monte Carlo sampling:}
		Draw a unsupervised minibatch $\mathcal{U} \sim \tilde{p}(\tilde{x}) p_\theta(x) q_\phi(h|x)$;
		
		\STATE \underline{SA updating:}
		
		Update $\theta$ by ascending:		
		$\frac{1}{|\mathcal{M}|} \sum_{(\tilde{x},x,h) \sim \mathcal{M}}
		\left[ \nabla_\theta u_\theta(\tilde{x}) - \nabla_\theta u_\theta(x) \right] $
		
		Update $\phi$ by ascending:
		$
		\frac{1}{|\mathcal{M}|} \sum_{(\tilde{x},x,h) \sim \mathcal{M}}
		\nabla_\phi logq_\phi(x,h) $
		
		\UNTIL{convergence}
	\end{algorithmic}
\end{algorithm}

\begin{prop}
	If Eq.(\ref{eq:jrf_unsup_gradient}) is solvable, then we can apply the SA algorithm to find its root.
\end{prop}

\begin{proof}
	This can be readily shown by recasting  Eq.(\ref{eq:jrf_unsup_gradient}) in the form of $f(\lambda) = 0$, with $\lambda \triangleq (\theta, \phi)^T$, $z \triangleq (\tilde{x}, x, h)$, $p(z; \lambda) \triangleq \tilde{p}(\tilde{x}) p_\theta(x) q_\phi(h|x)$, and
	
	\begin{displaymath}
	F(z; \lambda) \triangleq  \left( \begin{array}{c}
	\nabla_\theta u_\theta(\tilde{x}) - \nabla_\theta u_\theta(x)\\
	\nabla_\phi logq_\phi(x,h)
	\end{array} \right).
	\end{displaymath}
	According to Theorem \ref{theorem:SA}, the SA algorithm converges to a fixed point of Eq.(\ref{eq:jrf_unsup_gradient}).
\end{proof}

The inclusive generator can also be implemented as an autoregressive model, $q_\phi(x)$, like an LSTM-RNN \cite{asru}.
In this case, the learning algorithm is reduced to finding the root for the following system of simultaneous equations:
\begin{equation}
\label{eq:jrf_unsup_gradient2}
\left\{
\begin{split}
& E_{\tilde{p}(\tilde{x})}\left[\nabla_\theta logp_\theta(\tilde{x})\right]
=E_{\tilde{p}(\tilde{x})}\left[\nabla_\theta u_\theta(\tilde{x})\right]-E_{p_\theta(x)}\left[\nabla_\theta u_\theta(x)\right]=0\\
&E_{p_\theta(x)}\left[ \nabla_\phi logq_\phi(x)\right]=0\\
\end{split}
\right.
\end{equation}

To generate sample from the target random field $p_\theta(x)$, the auxiliary directed generator serves as the proposal for constructing the MCMC operator for the target random field, for which there are a number of choices. 
Metropolis Independence Sampling (MIS) is used in \cite{asru} which is appropriate for handling discrete observations, while SGLD/SGHMC is used in \cite{song2018learning} which is more efficient than MIS when handling continuous observations.

\subsubsection{Noise Contrastive Estimation (NCE) Learning}

Noise-contrastive estimation (NCE) is proposed in \cite{nce} for learning unnormalized statistical models.
Its basic idea is ``learning by comparison'', i.e. to perform nonlinear logistic regression to discriminate between data samples drawn from the data distribution $p_0(x)$ and noise samples drawn from a known noise distribution $p_n(x)$.
An advantage of NCE is that the normalization constants can be treated as the normal parameters and updated together with the model parameters.

To apply NCE to estimate the general random field defined in \ref{eq:unsup-RF}, we treat the logarithmic normalization constant $\log Z(\theta)$ as a parameter $c$ and rewrite \eqref{eq:unsup-RF} in the following form:
\begin{equation}\label{eq:trf-nce}
p_{\hat\theta}(x) = \exp\left[  \hat u_{\hat\theta}(x) \right] 
\end{equation}
Here $\hat u_{\hat\theta}(x) = u_\theta(x) - c$,
and $\hat\theta = (\theta, c)$ consists of the parameters of the potential function and the log normalization constant, which can be estimated together in NCE.
There are three distributions involved in NCE -- the true but unknown data distribution denoted by $p_0(x)$, the model distribution $p_{\hat\theta}(x)$ in \eqref{eq:trf-nce} and a fixed, known noise distribution denoted by $p_n(x)$.

Consider the binary classification of a sample $x$ coming from two classes -- from the data distribution ($C=0$) and from the noise distribution ($C=1$), where $C$ is the class label.
Assume that the ratio between the prior probabilities is $1:\nu$, and the class-conditional probability for $C=0$ is modeled by $p_{\hat\theta}(x)$.
Then the posterior probabilities can be calculated respectively as follows:
\begin{align}
P(C=0|x; \hat \theta) &= \frac{p_{\hat\theta}(x)}{p_{\hat\theta}(x) + \nu p_n(x)}  \label{eq:pc0} \\
P(C=1|x; \hat \theta) &= 1 - P(C=0|x; \hat \theta)  \label{eq:pc1}
\end{align}
NCE estimates the model distribution by maximizing the following conditional log-likelihood:
\begin{equation} \label{eq:NCE-obj}
\mathcal{J}(\hat \theta) = E_{x \sim p_0(x)} \left[ \log P(C=0|x;\hat\theta)\right]  + 
\nu E_{x \sim p_n(x)} \left[  \log P(C=1|x;\hat\theta) \right] 
\end{equation}
$\mathcal{J}(\hat\theta)$ is the summation of two expectations.
The first is the expectation w.r.t. the data distribution $p_0(x)$, which can be approximated by randomly selecting sentences from the training set.
The second is the expectation w.r.t. the noise distribution $p_n(x)$, which can be computed by drawing sentences from the noise distribution itself.

\begin{algorithm}[tb]
	\caption{\textbf{NCE learning}, represented as the SA algorithm with multiple moves}
	\label{alg:NCE}
	\begin{algorithmic}
		\REPEAT
		\STATE \underline{Monte Carlo sampling:}
		Draw an empirical minibatch $\mathcal{E} \sim p_0(x)$ and a noise minibatch $\mathcal{N} \sim p_n(x)$, satisfying $\nu = |\mathcal{N}|/|\mathcal{E}|$;
		\STATE \underline{SA updating:}
		Update $\hat\theta$ by ascending:\\
		$\frac{1}{|\mathcal{E}|} \sum_{x \sim \mathcal{E}}
		P(C=1|x;\hat\theta) \nabla_{\hat\theta} \hat u_{\hat\theta}(x) - \frac{\nu}{|\mathcal{N}|} \sum_{x \sim \mathcal{N}}
		P(C=0|x;\hat\theta) \nabla_{\hat\theta} \hat u_{\hat\theta}(x)$;
		\UNTIL{convergence}
	\end{algorithmic}
\end{algorithm}

Setting to zeros the gradients of $\mathcal{J}(\hat\theta)$ w.r.t. $\hat\theta$, we can apply the SA algorithm, as shown in Algorithm \ref{alg:NCE}, to find its root and thus solve the optimization problem Eq.(\ref{eq:NCE-obj}). The relevant gradients are
\begin{equation} \label{eq:NCE-grad}
\begin{split}
\nabla_{\hat\theta}\log P(C=0|x;\hat\theta) = 
P(C=1|x;\hat\theta) \nabla_{\hat\theta} \hat u_{\hat\theta}(x) \\
\nabla_{\hat\theta}\log P(C=1|x;\hat\theta) = 
P(C=0|x;\hat\theta) \nabla_{\hat\theta} \hat u_{\hat\theta}(x)
\end{split}
\end{equation}

It is shown in \cite{nce} that under the ideal situation of infinite amount of data and infinite capacity of $p_{\hat\theta}(x)$, we have the following theorem (Nonparametric estimation).
It is further shown \cite{nce} that the NCE estimator is consistent.
In Appendix \ref{sec:appendix-NCE}, we give an connection between NCE learning and variational estimation of $f$-divergence, which also shows that the data density $p_0(x)$ can be found by maximization of $\mathcal{J}(\hat\theta)$.

\begin{theorem}
	\label{theorem:NCE}
$\mathcal{J}(\hat\theta)$ attains its maximum at $p_{\hat\theta}(x) = p_0(x)$.
There are no other extrema if the noise density $p_n(x)$ is chosen such that it is nonzero whenever $p_0(x)$ is nonzero.
\end{theorem}	

Note that there exist two problems in applying NCE learning.
First, reliable NCE needs a large $\nu$, especially when the noise distribution is not close to the data distribution.
And the time and memory cost for gradient calculation are almost linearly increased with $\nu$.
Second, the expectation w.r.t. the data distribution in \eqref{eq:NCE-obj} is approximated by the expectation w.r.t. the empirical distribution (namely the training set), which is rather sparse in high-dimensionality tasks.
The model estimated by NCE is thus easily overfitted to the empirical distribution.
Dynamic noise-contrastive estimation (DNCE) is proposed in \cite{wang2018improved} to gracefully address the above two problems.

\section{Conclusion}

Deep generative models combines the representation powers of both graphical models and neural networks, and would play an important role in future artificial intelligence.
Learning with deep generative models, whether directed or undirected, is developing rapidly and making great advancement.
This review separates model definition and model learning, with more emphasis on reviewing, differentiating and connecting different existing learning algorithms. Systematic understanding of existing algorithms would be helpful for future research.
There are still some open questions, which are discussed in the previous sections and deserve more efforts.
\begin{itemize}
\item
Stabilize and speedup training, e.g. by exploring different training criteria, different regularizations, etc.
\item
Handling of discrete latent variables.
\item
Handling of discrete observations.
\item
Handling of sequential observations.
\end{itemize}

\section{Appendix A: The plugging trick: turning single-objective into multi-objective optimization} \label{sec:appendix-plugging}

Recently, there are increasing practices in turning (intractable) single-objective optimization into (tractable) multi-objective optimization through plugging \citep{mescheder2017adversarial,Pu2017SymmetricVA}.
For example, for an objective function which involves an intractable likelihood ratio term, plugging the likelihood ratio estimator given by adversarial learning would yield a tractable two-objective optimization problem.
In the following, we present this plugging trick with some caution.

Suppose that the target optimization problem is given by
\begin{equation} \label{eq:plugging-single-obj}
	\max_{\phi} F(\phi, V(\phi)),
\end{equation}
where we use $V(\phi)$ to collect the intractable term which appears in the definition of the target objective function. 
If we can represent $V(\phi)$ by using the solution from another optimization problem,
\begin{equation} \label{eq:plugging}
V(\phi) = v(\phi, \psi_\phi^*), \psi_\phi^* = \argmax_{\psi} f(\phi, \psi),
\end{equation}
Then we can turn the single-objective optimization problem Eq. \ref{eq:plugging-single-obj} into the following multi-objective problem:
\begin{equation} \label{eq:plugging-two-obj}
\left\{
\begin{split}
&\max_{\phi} F(\phi, v(\phi, \psi)) \\
&\max_{\psi} f(\phi,\psi)
\end{split}
\right.
\end{equation}

Through plugging, we transform the optimization Eq. \ref{eq:plugging-single-obj} into Eq. \ref{eq:plugging-two-obj}. 
Plugging itself only guarantees the equality of the objectives, namely 
\begin{displaymath}
\left. F(\phi, v(\phi, \psi)) \right|_{\psi=\psi^*_\phi} =  F(\phi, V(\phi))
\end{displaymath}
for the optimal $\psi^*_\phi$, determined by Eq. \ref{eq:plugging}.
\textbf{For the plugging trick to be valid}, we need to additionally check the gradients
\begin{displaymath}
\left. \left\lbrace \nabla_{\phi} F(\phi, v(\phi, \psi)) \right\rbrace \right|_{\psi=\psi^*_\phi} = \nabla_{\phi} F(\phi, V(\phi)),
\end{displaymath}
because in optimizing the first line in Eq. \ref{eq:plugging-two-obj}, we ignore the implicit dependence of $v(\phi, \psi)$ on $\phi$, i.e. we take the gradient of $F(\phi, v(\phi, \psi))$ w.r.t. $\phi$ while treating $\psi$ as constant.
This check is equivalent to examine that the gradient of $F$ due to such ignorance is zero: 
\begin{displaymath}
\left. \frac{\partial F(\phi, v(\phi, \psi))}{\partial \psi}\right| _{\psi=\psi^*_\phi}  \frac{\partial \psi^*_\phi}{\partial \phi}  = 0.
\end{displaymath}

\section{Appendix B: Connection between NCE and Variational Estimation of $f$-divergence\footnote{Provided by Bin Wang.}} \label{sec:appendix-NCE}

Denote by $p_m(x; \psi)$ the model distribution needed to be estimated\footnote{Here $p_{\hat\theta}(x)$ is denoted as $p_m(x; \psi)$, to make the connection explicit.}.
Denote by $p_n(x)$ a known noise distribution, from which it is easy to draw samples.
Then we can rewrite Eq. \ref{eq:f-gan-bound} by replacing $p_\theta$ to $p_n$:
\begin{equation}\label{eq:f-div-nce}
  D_f[p_0||p_n] \geq \sup_{T \in \mathcal{T}} \left(E_{p_0(x)} \left[T(x;\psi)\right] - E_{p_n(x)} \left[f^\dagger(T(x;\psi))\right]\right)
\end{equation}
where $T(x;\psi)$ with parameter $\psi$ and the $f$-function are defined as follow:
\begin{displaymath}
\begin{split}
f(u) &\triangleq u\log u - (u + \nu) \log (u + \nu) \\
f'(u) &= \log \frac{u}{u+\nu} \\
f^\dagger(t) &= \sup_{u} \{ ut - f(u) \} = \nu \log \nu - \nu \log (1 - e^t) \\
T(x; \psi) &\triangleq f'\left( \frac{p_m(x; \psi)}{p_n(x)} \right) = \log \frac{p_m(x; \psi)}{p_m(x; \psi) + \nu p_n(x)}
\end{split}\label{eq:f-div-nce-f}
\end{displaymath}

According to Eq. \ref{eq:f-div-nce}, $D_f[p_0||p_n]$ can be estimated by maximizing the following objective function (the variational lower bound) w.r.t. $\psi$:
\begin{equation}\label{eq:f-div-nce-obj}
\begin{split}
  \mathcal{F}(\psi) &= E_{p_0(x)} \left[T(x; \psi)\right] - E_{p_n(x)} \left[f^\dagger(T(x; \psi))\right] \\
          &= E_{p_0(x)} \log \frac{p_m(x; \psi)}{p_m(x; \psi) + \nu p_n(x)} + \nu E_{p_n(x)} \log \frac{\nu p_n(x)}{p_m(x; \psi) + \nu p_n(x)} + \nu \log \nu \\
\end{split}
\end{equation}
Ignoring the constant item $\nu \log \nu$, Eq. \ref{eq:f-div-nce-obj} is the objective function of NCE.

According to Eq. \ref{eq:f-gan-optimal-T}, the optimal $p_m(x; \psi)$ satisfies $p^*_m = p_0$.

\bibliography{ozj}
\bibliographystyle{iclr2016_workshop}

\begin{landscape}
	\begin{table}[t]
		\caption{Summary of learning algorithms for directed generative models.
			We organize the summary according to the objection functions used in joint training of the target model $p_\theta(x,h)$ or $p_\theta(x)$, and the auxiliary model $q_\phi(h|x)$ or $V_\psi(x)$. 
			Shorthands: ML (maximum likelihood), V-LB (variational lower bound), IW-LB (importance-weighting lower bound),
			$*$ (using the reparameterization trick for continuous $h$, which is not feasible for discrete $h$), $\diamond$ (using variance-reduction techniques for discrete $h$), $\times$ (not applicable).
		}
		\label{table:directed-learning-summary}
		\begin{center}
			\begin{tabular}{c|c|c|c|c}
				-
				&\makecell{ ML\\ $\max_\theta E_{\tilde{p}(x)} log p_\theta(x)$ }
				&\makecell{ V-LB\\$\max_\theta E_{\tilde{p}(x)} \mathcal{L}(x;\theta,\phi)$}
				&\makecell{ IW-LB\\$\max_\theta E_{\tilde{p}(x)} \mathcal{I}_K(x;\theta,\phi)$}
				&\makecell{$f$-divergence-LB\\$min_\theta \mathcal{F}(\theta,\psi)$}\\
				\hline
				\hline	
				\makecell{ V-LB\\$\max_\phi E_{\tilde{p}(x)} \mathcal{L}(x;\theta,\phi)$\\$\Leftrightarrow \min_\phi E_{\tilde{p}(x)} KL\left[ q_\phi(h|x) || p_\theta(h|x) \right]$} 
				&$\times$ 
				&\makecell{VAE$^*$ \\NVIL$^{\diamond}$\\MuProp} &$\times$
				&$\times$\\
				\hline
				\makecell{ IW-LB\\$\max_\phi E_{\tilde{p}(x)} \mathcal{I}_K(x;\theta,\phi)$ }
				&$\times$ 
				&$\times$ 
				&IWAE$^*$ 
				&$\times$\\
				\hline
				$\min_\phi E_{\tilde{p}(x)} KL\left[ p_\theta(h|x) || q_\phi(h|x) \right]$ IS
				&$\times$ 
				&$\times$ 
				&RWS 
				&$\times$\\
				\hline
				$\min_\phi E_{\tilde{p}(x)} KL\left[ p_\theta(h|x) || q_\phi(h|x) \right]$ MIS
				&JSA 
				&$\times$ 
				&$\times$ 
				&$\times$\\
				\hline	
				\makecell{sleep phase $\phi$-update\\$\min_\phi E_{p_\theta(x)} KL\left[ p_\theta(h|x) || q_\phi(h|x) \right]$}
				& 
				&WS 
				&RWS
				&$\times$\\
				\hline	
				\makecell{$f$-divergence-LB\\$min_\psi \mathcal{F}(\theta,\psi)$}
				&$\times$ 
				&$\times$
				&$\times$
				&Adversarial Learning\\							
			\end{tabular}
		\end{center}
		Note:
		\begin{itemize}
			\item
			Since IW-LB is a tighter lower bound than V-LB, the crossing use of the two bounds in optimizing $\theta$ and $\phi$ does not make sense (which is marked with $\times$), though possible.
			\item
			For $\min_\phi E_{\tilde{p}(x)} KL\left[ p_\theta(h|x) || q_\phi(h|x) \right]$, you could choose either IS or MIS.
			\item
			If we have obtained samples from $p_\theta(h|x)$ through MIS accept/reject, then it would be unnecessary to update $\theta$ according to V-LB or update $\phi$ according to V-LB or IW-LB.
			\item
			It is reported in \cite{bornschein2014reweighted} that combining wake and sleep phase $\phi$-updates generally gives the best
			results. It is possible to add sleep phase $\phi$-update to JSA.
		\end{itemize}
	\end{table}
\end{landscape}
\end{document}